%% file: neurips_2026.tex
\documentclass{article}

\PassOptionsToPackage{table}{xcolor}

\usepackage[preprint]{neurips_2026}
\usepackage{subcaption}
\usepackage{pdfpages}
\usepackage{wrapfig}
\usepackage{float}

\usepackage[utf8]{inputenc} % allow utf-8 input
\usepackage[T1]{fontenc} % use 8-bit T1 fonts
\usepackage{hyperref}  % hyperlinks
\usepackage{url}  % simple URL typesetting
\usepackage{booktabs}  % professional-quality tables
\usepackage{amsfonts}  % blackboard math symbols
\usepackage{nicefrac}  % compact symbols for 1/2, etc.
\usepackage{microtype}  % microtypography

\usepackage{mathtools} % amsmath with fixes and additions (needed for \coloneqq)
\usepackage{tikz} % nice language for creating drawings and diagrams
\usepackage{amssymb} % needed for \mathbb, \rightsquigarrow, etc.
\usepackage{algorithm}
\usepackage{algorithmic}
\usepackage{tabularx}
 \usepackage{multirow}
 \usepackage{amsthm}

\usepackage{scalerel}  % Required for the \scaleto command used in the confidence intervals
\usepackage{hyperref} % Required for \href in author block and clickable links
\usepackage{cleveref}

\definecolor{ourblue}{HTML}{2B6CBC}  % Deep blue for text
\definecolor{ourlightblue}{HTML}{E9F1F8} % Very light blue for column backgrounds
\definecolor{ourmiddle}{HTML}{D95319}  % Accent color (often burnt orange/red) for "Ours"
\definecolor{ourlightmiddle}{HTML}{FDF0EB} % Very light accent for column backgrounds
\hypersetup{
 colorlinks=true,
 citecolor=ourmiddle, % Colors the inline reference names (e.g., Smith et al., 2024)
 linkcolor=ourblue, % Colors internal document links (e.g., Figure 1, Equation 3)
 urlcolor=ourblue  % Colors external web URLs
}

\newtheorem{theorem}{Theorem}
\newtheorem{proposition}[theorem]{Proposition}

\usepackage[many]{tcolorbox}

\tcolorboxenvironment{theorem}{
 enhanced,
 breakable,   % Allows the box to break across pages/columns
 colback=ourlightmiddle, % Background color
 frame hidden,  % Removes the border for a cleaner, modern look
 boxrule=0pt,  % Border thickness
 top=2mm, bottom=2mm, left=2mm, right=2mm, % Inner padding
 sharp corners  % Use 'arc=2mm' if you prefer rounded corners
}

\tcolorboxenvironment{proposition}{
 enhanced,
 breakable,
 colback=ourlightmiddle, 
 frame hidden,
 boxrule=0pt,
 top=2mm, bottom=2mm, left=2mm, right=2mm,
 sharp corners
}

\title{Entropy-Regularized Adjoint Matching for Offline Reinforcement Learning}

\author{%
  Abdelghani Ghanem \quad\quad Mounir Ghogho \\
  College of Computing\\
  Mohammed VI Polytechnic University\\
  Rabat 11103, Morocco \\
  \texttt{\{abdelghani.ghanem-ext, mounir.ghogho\}@um6p.ma}
}

\begin{document}

\maketitle

\begin{abstract}
Integrating expressive generative policies, such as flow-matching models, into offline reinforcement learning (RL) allows agents to capture complex, multi-modal behaviors. While Q-learning with Adjoint Matching (QAM) stabilizes policy optimization via the continuous adjoint method, it remains inherently bound to the fixed behavior distribution. This dependence induces a \textit{popularity bias} that can suppress high-reward actions in low-density regions, and creates a \textit{support binding} that restricts off-manifold exploration. Existing workarounds, such as appending \textit{residual} Gaussian policies, often re-introduce the expressivity bottlenecks associated with unimodal distributions. In this work, we propose \textit{Maximum Entropy Adjoint Matching} (ME-AM), a unified framework that addresses these limitations within the continuous flow formulation. ME-AM incorporates two mechanisms: (1) a Mirror Descent entropy maximization objective that mitigates the popularity bias to facilitate the extraction of optimal policies from offline datasets, and (2) a \textit{Mixture Behavior Prior} that broadens the geometric support to encompass out-of-distribution high-reward regions. By exploring this extended geometry, ME-AM identifies robust actions while preserving the absolute continuity of the generative vector field. Empirically, ME-AM demonstrates competitive or superior performance compared to prior state-of-the-art (SOTA) methods across a diverse suite of sparse-reward continuous control environments.

\textbf{Keywords:} Offline Reinforcement Learning, Flow Matching, Adjoint Matching, Maximum Entropy.
\end{abstract}
\section{Introduction}

\begin{figure}[t]
  \centering
  
  % --- LEFT SIDE: TIKZ DIAGRAM ---
  \begin{minipage}[c]{0.56\textwidth}
   \centering
   % TWEAK HEIGHT HERE: Change '4.5cm' to adjust the maximum height
   \resizebox{!}{4.5cm}{\input{figs/method_overview.tex}}
  \end{minipage}\hfill
  % --- RIGHT SIDE: AGGREGATE PLOT ---
  \begin{minipage}[c]{0.44\textwidth}
   \centering
   % TWEAK HEIGHT HERE: Change '4.5cm' to match the left side
   \includegraphics[width=\linewidth, height=4.5cm, keepaspectratio]{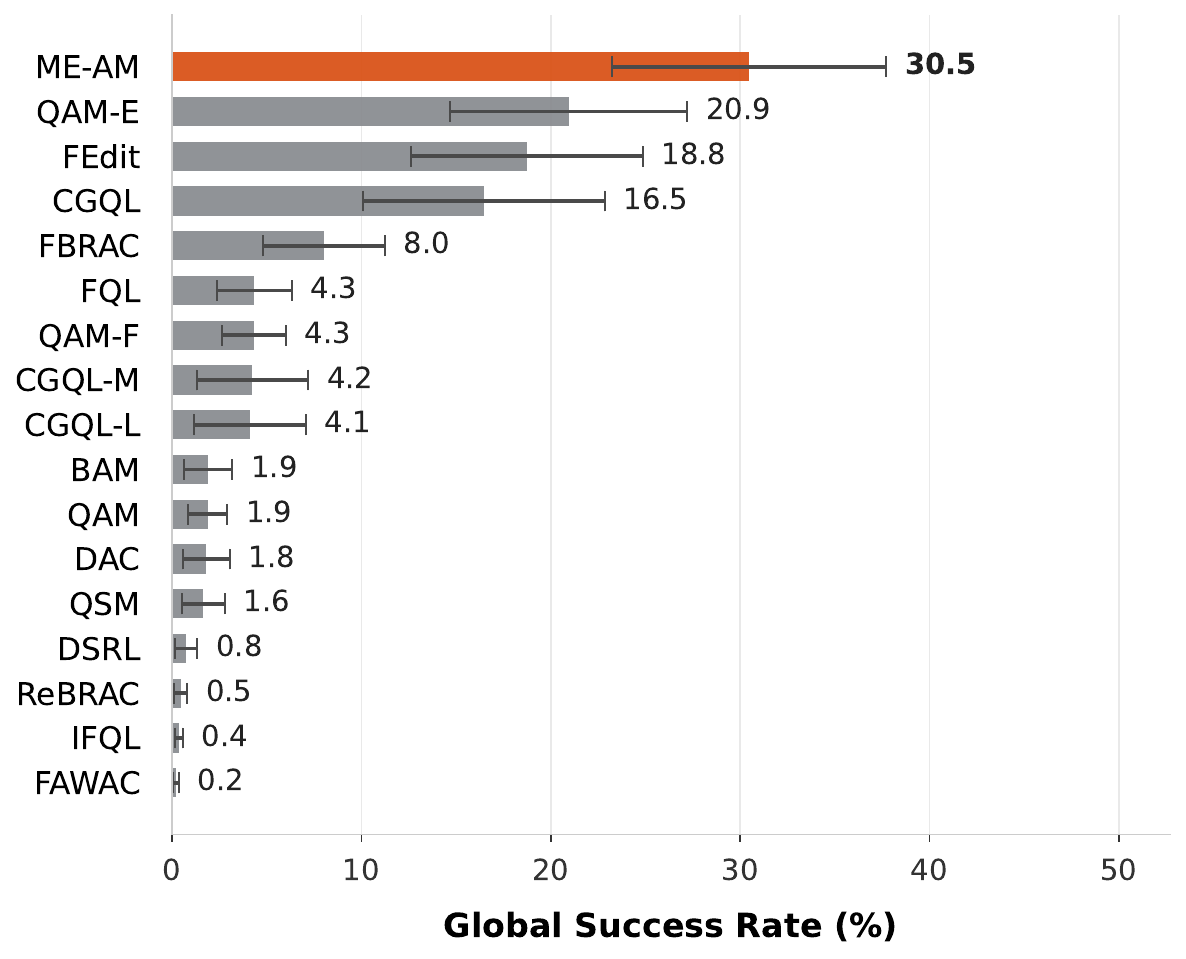}
  \end{minipage}

  % --- COMBINED MASTER CAPTION ---
  \caption{\textbf{ME-AM: Maximum Entropy Adjoint Matching.} \textbf{Left:} While prior methods rely on disjointed edits to reach out-of-distribution optimal actions, ME-AM expands the geometric support and uses entropy maximization \citep{de2025provable} to reach high-reward regions via a unified, continuous flow. \textbf{Right:} Aggregate offline RL performance (2 domains on OGBench \citep{park2025ogbench}, 10 tasks, 8 seeds).}
  \label{fig:teaser_overview}
\end{figure}

The integration of continuous-time generative models, such as flow matching and diffusion, has advanced offline RL by enabling agents to capture expressive, multi-modal behavior distributions \citep{ren2024diffusion, fang2025diffusion, edp_kang2023, dtql_chen2024, srpo_chen2024, qgpo_lu2023, qvpo_ding2024, dql_wang2023, diffcps_he2023, consistencyac_ding2024, srdp_ada2024, entropydql_zhang2024, hansen2023idql}. However, optimizing these generative policies to maximize a critic's reward introduces a critical tension: the agent must step off the empirical data manifold to discover novel optimal actions, yet it must remain safely anchored to the offline dataset to prevent catastrophic overestimation. While recent continuous-time Optimal Control frameworks, such as Adjoint Matching \cite{li2026qlearning, domingo-enrich2025adjoint}, have stabilized this optimization process by bypassing unstable solver backpropagation, strict adherence to the data manifold inadvertently traps the agent.

This tension between maximizing reward and avoiding out-of-distribution (OOD) extrapolation has deep roots in offline RL \citep{fujimoto2019off, kumar2019stabilizing, levine2020offline}. Extracting an optimal policy fundamentally requires giving sufficient advantage to the critic to filter out sub-optimal behavior. A straightforward approach, common in KL-regularized and weighted-regression literature \citep{peters2010relative, peng2019advantage, nair2020awac, wang2020critic, kostrikov2021offline}, relies on temperature scaling to emphasize high-reward actions ($\pi^* \propto \pi_\beta \exp(Q/\tau)$). However, this strict multiplicative dependence on the behavior prior creates a structural bottleneck that we formally characterize as the \textit{Support-Binding Dilemma} (SBD), paralyzing the agent across two distinct dimensions: density and geometry. On the density front, it induces a \textit{popularity bias}: to overcome the suppression of a rare expert action by $\pi_\beta$, the temperature must be aggressively tuned. This exponential scaling amplifies the critic's function approximation errors, forcing historical methods to dampen the contrast (e.g., via weight clipping) and accept sub-optimal, high-frequency actions just to maintain stability. On the geometric front, it creates a \textit{zero-support trap}: if an optimal action lies outside the behavioral support, $\pi_\beta = 0$. The constraint forbids the agent from ever reaching this region, halting the discovery of novel behaviors regardless of the reward signal. Modern continuous-time frameworks like Adjoint Matching \cite{li2026qlearning} seamlessly inherit this exact dual pathology.

To circumvent the geometric limitation of the zero-support trap, current  SOTA continuous-time frameworks typically employ residual \textit{edit} strategies or post-processing heuristics \cite{li2026qlearning, yuan2024policy, dong2026expo}. By appending a learnable Gaussian policy to the generative model's terminal output, these methods attempt to force the agent off the valid data manifold to reach novel optimal modes. This heuristic assumes that the base flow can navigate to the general vicinity of an optimal action, leaving the Gaussian edit to perform local refinement. However, in regions devoid of data support, the base flow yields sub-optimal actions; a local Gaussian edit cannot bridge the geometric gap to distant optimal modes without injecting, destabilizing variance. Furthermore, by forcing a stochastic, unimodal jump at the final integration step, terminal edits act as a structural patch that breaks the continuous Ordinary Differential Equation (ODE) formulation. This disrupts the coherent vector field trajectory and re-introduces the exact expressivity bottlenecks that generative models were designed to overcome.

In this work, we propose \textit{Maximum Entropy Adjoint Matching (ME-AM)}, a novel framework that solves the SBD entirely within the continuous flow formulation. Our approach addresses both the geometric and density limitations simultaneously: (i) \textit{Geometric Expansion via a Reward-Guided Mixture Prior:} Rather than appending Gaussian edits at inference time to escape the zero-support trap, we introduce a learnable, critic-guided expansion directly into the flow's \textit{training targets}. By augmenting the behavior dataset with Q-maximized actions, we extend the valid geometric support. The model natively learns to transport noise to these out-of-distribution modes, navigating the void between data clusters while maintaining a structurally sound, continuous ODE. (ii) \textit{Density Flattening via Mirror Descent:} To overcome popularity bias, we integrate an amortized Maximum Entropy (MaxEnt) objective directly into the adjoint system \cite{de2025provable, hsieh2019finding}. By executing functional Mirror Descent in continuous time, ME-AM aligns the generative model's score function with a uniform target \cite{hazan2019maxent}, exponentially flattening the peaked behavior density across the extended support. Crucially, this density flattening allows the agent to emphasize high-reward actions without requiring the extreme, error-amplifying contrast of strict temperature scaling.

Our core contribution is the ME-AM framework,which introduces a Reward-Guided Mixture Prior and functional Mirror Descent to overcome both the geometric and density limitations of offline datasets. Furthermore, we provide theoretical guarantees establishing that this continuous-time objective converges exactly to the targeted density-flattened stationary policy. Finally, empirical evaluations on challenging, sparse-reward continuous control benchmarks demonstrate that ME-AM achieves state-of-the-art offline performance and maintains robust sample efficiency during offline-to-online fine-tuning.
\section{Related Work}
\label{sec:related}

\textbf{Generative Policy Optimization and the Support-Binding Problem.} While Adjoint Matching \cite{chen2018neural,domingo-enrich2025adjoint,li2026qlearning} bypasses solver differentiation, its strict regularization against the behavior density introduces a zero-support trap, where learning gradients effectively vanish in OOD regions devoid of data. To facilitate OOD exploration, methods often append terminal Gaussian edits (e.g., QAM-Edit) \cite{li2026qlearning,yuan2024policy,dong2026expo}. While concurrent work avoids inference-time edits by expanding flow coverage via verifier-constrained entropy maximization \cite{deverifier}, it relies on alternating projection steps against a domain-specific oracle. In contrast, ME-AM achieves geometric expansion natively during training, driven entirely by the standard offline critic.

\textbf{Weighted Regression, Popularity Bias, and the Density Trap.}
Many KL-regularized offline RL algorithms, such as REPS \cite{peters2010relative}, AWAC \cite{nair2020awac}, and CRR \cite{wang2020critic}, converge to an optimal behavior-weighted Boltzmann policy: $\pi^*(a|s) \propto \pi_\beta(a|s) \exp(Q(s, a)/\tau)$. Continuous-time Adjoint Matching \cite{domingo-enrich2025adjoint,li2026qlearning} is mathematically shown to target a similar distribution. However, this multiplicative dependence on the prior introduces a \textit{popularity bias}. If an optimal action is rare in the dataset, the vanishing $\pi_\beta$ term suppresses the exponential reward signal. Because standard weighted regression methods project a neural policy onto $\pi^*$ via Forward Kullback-Leibler (KL) divergence using importance sampling, the $\pi_\beta$ density structurally cancels out of the loss function. This prevents explicitly flattening the density without introducing unstable explicit estimators. 

\textbf{MaxEnt and Distributional Control.}
While MaxEnt RL \cite{hazan2019maxent} is conventionally used to encourage online exploration, ME-AM leverages it to resolve the offline popularity bias. By maximizing the differential entropy of the policy, the generative model flattens its density across the extended mixture manifold, achieving the necessary tempering that prior weighted-regression methods cannot explicitly enforce. Applying this distributional objective to generative flows, however, is challenging. Evaluating the exact differential entropy requires integrating the trace of the Jacobian over time, which is computationally expensive in high-dimensional action spaces \cite{chen2018neural}. To bypass explicit density evaluation, frameworks like EXPO \cite{dong2026expo} append a Gaussian edit policy to a base generative model to permit standard closed-form entropy regularization. While convenient, this restricts entropy maximization to a unimodal residual primarily designed for local online exploration, leaving the macroscopic density of the expressive base flow unchanged. To achieve entropy maximization directly within the generative model, recent literature formulates it as sequential fine-tuning via Functional Mirror Descent \cite{de2025provable,hsieh2019finding}, where the intractable entropy gradient simplifies to the negative score function. ME-AM synthesizes these approaches by embedding this score-based entropy gradient directly into the terminal boundary condition of the Adjoint Matching ODE.

\section{Preliminaries}
\label{sec:preliminaries}

We consider the standard RL setting formulated as a Markov Decision Process (MDP) $\mathcal{M} = (\mathcal{S}, \mathcal{A}, P, R, \gamma, \rho_0)$. Here, $\mathcal{S} \subseteq \mathbb{R}^{d_s}$ and $\mathcal{A} \subseteq \mathbb{R}^{d_a}$ represent the continuous state and action spaces, respectively. The function $P: \mathcal{S} \times \mathcal{A} \to \Delta_{\mathcal{S}}$ defines the transition dynamics (where $\Delta_{\mathcal{S}}$ denotes the set of probability distributions over $\mathcal{S}$), $R: \mathcal{S} \times \mathcal{A} \to \mathbb{R}$ is the scalar reward function, $\gamma \in [0, 1)$ is the discount factor, and $\rho_0 \in \Delta_{\mathcal{S}}$ is the initial state distribution. We assume access to a static offline dataset $\mathcal{D} = \{(s_i, a_i, r_i, s'_i)\}_{i=1}^{|\mathcal{D}|}$ generated by an unknown behavior policy $\pi_\beta(\cdot|s)$, where $s'_i \sim P(\cdot | s_i, a_i)$ and $r_i = R(s_i, a_i)$.

Our primary goal (\textit{offline RL}) is to learn a policy $\pi(\cdot|s): \mathcal{S} \to \Delta_{\mathcal{A}}$ from the dataset $\mathcal{D}$ that maximizes the expected discounted return:
$$\mathcal{J}(\pi) = \mathbb{E}_{s_0 \sim \rho_0, a_k \sim \pi(\cdot|s_k), s_{k+1} \sim P(\cdot|s_k, a_k)} \left[ \sum_{k=0}^\infty \gamma^k R(s_k, a_k) \right].$$
In the context of offline RL, we typically constrain the learned policy to the support of the empirical behavior distribution to avoid evaluating out-of-distribution actions that can trigger severe overestimation bias \cite{kumar2019stabilizing,li2026qlearning}. As a secondary goal (\textit{offline-to-online RL}), we aim to sample-efficiently fine-tune this pre-trained policy by subsequently interacting with the MDP.

\textbf{Conditional Flow Matching.} Rather than explicitly parameterizing the policy $\pi(\cdot|s)$, we model it implicitly using Conditional Flow Matching (CFM) \cite{lipman2023flow}. The policy is induced by a generative ODE that transports standard Gaussian noise $x_0 \sim p_0=\mathcal{N}(0, I_{d_a})$ to the data distribution $p_1 \approx \pi_\beta(\cdot|s)$ over a continuous time horizon $t \in [0, 1]$:
\begin{equation}
 \frac{dx_t}{dt} = v_{\theta_{\text{base}}}(x_t, t \mid s), \quad x_0 \sim p_0,
\end{equation}
where $x_t \in \mathbb{R}^{d_a}$ is a latent action variable, and $v_{\theta_{\text{base}}} : \mathbb{R}^{d_a} \times [0,1] \times \mathcal{S} \to \mathbb{R}^{d_a}$ is a learnable neural vector field conditioned on the state $s \in \mathcal{S} \subseteq \mathbb{R}^{d_s}$. Throughout the paper, we use uppercase $X_t$ to denote the random variable at time $t$, and lowercase $x_t$ to denote its realization. We adopt an \textit{Independent Linear} probability path, which corresponds to a conditional vector field $u_t(x_t \mid x_1) = x_1 - x_0$, where $u_t(\cdot \mid x_1) : \mathbb{R}^{d_a} \to \mathbb{R}^{d_a}$ and $x_1 \in \mathbb{R}^{d_a}$ denotes the target action.

To bypass intractable marginalization, CFM regresses $v_{\theta_{\text{base}}}$ directly against this conditional path. Given an offline dataset $\mathcal{D}$ with empirical actions $a \in \mathbb{R}^{d_a}$, the model is trained via:
\begin{equation}
\label{CFM_loss}
 \mathcal{L}_{\text{CFM}}(\theta_{\text{base}}) = \mathbb{E}_{t \sim \mathcal{U}(0,1), (s, a) \sim \mathcal{D}, x_0 \sim p_0} \left[ \| v_{\theta_{\text{base}}}(x_t, t \mid s) - (a - x_0) \|^2 \right].
\end{equation}

\textbf{Adjoint Matching.} Optimizing the objective in Eq.~\ref{CFM_loss} yields a behavioral cloning policy that matches the empirical action distribution. To instead obtain a $Q$-value maximizing policy, we formulate the problem as a conditional continuous-time Stochastic Optimal Control (SOC) problem for a given state $s$ \cite{kappen2005path, rawlik2013stochastic, domingo-enrich2025adjoint}. In this formulation, the policy is parameterized by a fine-tuned vector field $v_{\theta_{\text{fine}}}$, which steers the generative process toward high-value actions under the critic $Q_\phi(s,a)$, where $Q_\phi : \mathcal{S} \times \mathbb{R}^{d_a} \to \mathbb{R}$ is a parametric value function with parameters $\phi$. The generative process is defined by the following memoryless Stochastic Differential Equation (SDE):
\begin{equation}
 dX_t = \left(2 v_{\theta_{\text{fine}}}(X_t, t \mid s) - X_t/t\right) dt + g_t \, dW_t,
\end{equation}
where $g_t = \sqrt{2(1-t)/t}$ and $(W_t)_{t \in [0,1]}$ is a standard $d_a$-dimensional Brownian motion. This choice of diffusion ensures that the initial noise $X_0$ and the terminal sample $X_1$ are independent, which is required for recovering the correct tilted distribution under the SOC formulation \citep{domingo-enrich2025adjoint}.

To maximize $Q_\phi(s,a)$ while preventing the policy from collapsing into out-of-distribution modes, the optimization is constrained by penalizing the deviation of the fine-tuned trajectory from the base trajectory. Let $C([0,1], \mathbb{R}^{d_a})$ denote the space of continuous trajectories. We define $\mathbb{P}_{\text{base}}(\cdot \mid X_0, s)$ and $\mathbb{P}_{\text{fine}}(\cdot \mid X_0, s)$ as probability measures over $C([0,1], \mathbb{R}^{d_a})$ induced by $v_{\theta_{\text{base}}}$ and $v_{\theta_{\text{fine}}}$, respectively. Under the Memoryless Flow Matching formulation \cite{domingo-enrich2025adjoint}, the KL divergence between these path measures admits the following quadratic form:
\begin{equation}
   D_{\mathrm{KL}}(\mathbb{P}_{\text{fine}}(\cdot \mid X_0, s) \parallel \mathbb{P}_{\text{base}}(\cdot \mid X_0, s)) = \mathbb{E}_{\boldsymbol{X} \sim \mathbb{P}_{\text{fine}}} \left[ \int_0^1 \frac{2}{g_t^2} \left\| v_{\theta_{\text{fine}}}(X_t, t \mid s) - v_{\theta_{\text{base}}}(X_t, t \mid s) \right\|^2 dt \right],
 \label{eq:girsanov_kl}
\end{equation}
where $\boldsymbol{X} = (X_t)_{t \in [0,1]}$ is a random trajectory. We note that this expression differs from the control-based KL formulation in \cite{domingo-enrich2025adjoint}. The equivalence between the two under our parameterization is derived in Appendix~\ref{sec:appendix_soc}.

Directly optimizing the SOC objective requires backpropagation through the stochastic dynamics, which is unstable in practice. Instead, we adopt Adjoint Matching \cite{domingo-enrich2025adjoint, li2026qlearning}, which introduces a lean adjoint state $\tilde{y}(X_t, t) \in \mathbb{R}^{d_a}$ evolving backward in time as:
\begin{equation}
 \frac{d}{dt} \tilde{y}(X_t, t) = - \nabla_{X_t} v_{\theta_{\text{base}}}(X_t, t \mid s)^\top \tilde{y}(X_t, t),
 \label{eq:lean_adjoint_ode}
\end{equation}
with terminal condition given by $\tilde{y}(x_1, 1) = -\frac{1}{\tau(s)}\,\nabla_{x_1} Q_\phi(s, x_1)$, where $x_1$ denotes the terminal sample and $\tau : \mathcal{S} \to \mathbb{R}_{>0}$ is a temperature parameter controlling the strength of the value maximization. In practice, we use a state-independent constant $\tau$ \citep{li2026qlearning}.

The fine-tuned flow is then optimized via the adjoint matching loss:
\begin{equation}
 \mathcal{L}_{\mathrm{AM}}(\theta_{\text{fine}}) = \mathbb{E}_{\boldsymbol{X} \sim \mathbb{P}_{\text{base}}, s \sim \mathcal{D}} \left[ \int_0^1 \left\| \frac{2}{g_t} \left( v_{\theta_{\text{fine}}}(X_t, t \mid s) - v_{\theta_{\text{base}}}(X_t, t \mid s) \right) + g_t \tilde{y}(X_t, t) \right\|^2 dt \right],
 \label{eq:am_loss}
\end{equation}
which aligns the fine-tuned vector field with the optimal control direction implied by the adjoint dynamics, and induces a policy $\pi(\cdot \mid s)$ given by the terminal distribution of the process, satisfying $\pi(\cdot \mid s) \propto \pi_\beta(\cdot \mid s)\exp\big(\frac{1}{\tau(s)}\,Q_\phi(s,\cdot)\big)$.

\textbf{Entropy Maximization via Mirror Descent.} To overcome the popularity bias of the offline dataset, we aim to maximize the differential entropy of the fine-tuned policy, characterized by the terminal density $p^{\text{fine}}_1(\cdot|s)$ induced by $v_{\theta_{\text{fine}}}$. Evaluating this exact entropy is computationally prohibitive, as it requires integrating the intractable trace of the vector field's Jacobian \cite{chen2018neural}. To bypass explicit density evaluation, we reframe the problem using Functional Mirror Descent (MD) \cite{de2025provable}. Instead of directly optimizing the exact entropy, MD maximizes a linearized surrogate regularized by a KL divergence penalty. For our fine-tuned density, the effective functional gradient driving this maximization simplifies exactly to the negative score evaluated at the terminal sample $x_1$:
\begin{equation}
 \nabla_{x_1} \left( \frac{\delta \mathcal{H}(p^{\text{fine}}_1)}{\delta p}[x_1] \right) = -\nabla_{x_1} \log p^{\text{fine}}_1(x_1|s).
\end{equation}

Crucially, enforcing the MD KL penalty does not require an explicit optimization constraint. As established above, our continuous-time SOC formulation inherently penalizes the KL divergence via the cumulative control cost (Equation \ref{eq:girsanov_kl}). Therefore, we seamlessly execute this entropy maximization step within the Adjoint Matching framework by injecting the negative score directly into the terminal adjoint state alongside the value gradient.

\textbf{The Boundary Instability of the Score.} Our approach relies on evaluating the score $\nabla_{x_t} \log p^{\text{fine}}_t(x_t|s)$ at the terminal boundary $t \to 1$, where $p^{\text{fine}}_t(\cdot \mid s)$ denotes the conditional density of $X_t$, induced by $v_{\theta_{\text{fine}}}$. While this quantity is analytically linked to the marginal vector field \cite{domingo-enrich2025adjoint,holderrieth2025introductionflowmatchingdiffusion,de2025flow}, its practical estimation is numerically unstable: the underlying formulation involves a $(1-t)^{-1}$ factor, causing neural approximation errors in $v_{\theta_{\text{fine}}}$ to be amplified as $t \to 1$ \cite{hu2025improving}. We provide a detailed discussion of this boundary instability in Appendix \ref{sec:appendix_extended_background}. To circumvent this issue, we adopt a decoupled auxiliary score parameterization (Section \ref{sec:method}).

\section{Maximum Entropy Adjoint Matching (ME-AM)}
\label{sec:method}

To bypass the zero-support trap discussed in Section \ref{sec:related} without structural compromise, ME-AM shifts the constraint from the strict empirical density to a tempered density over an extended geometric support. Our objective is to learn a flow-induced policy $\pi^*$ governed by:
\begin{equation}
 \pi^*(a|s) \propto \left( \pi_{\text{mix}}(a|s) \right)^{\delta(s)} \exp\left( \frac{1}{\kappa(s)} Q_\phi(s,a) \right),
\end{equation}
where $\pi_{\text{mix}}$ is a constructed mixture distribution that geometrically bridges isolated data clusters, $\kappa: \mathcal{S} \to \mathbb{R}_{>0}$ is the state-dependent effective reward temperature, and $\delta: \mathcal{S} \to (0, 1)$ actively flattens the prior density to overcome popularity bias.

\paragraph{The Reward-Guided Mixture Prior.}
\label{sec:method_mixture}

Standard Adjoint Matching confines the policy to the behavior support $\pi_\beta$, causing the analytical gradient to vanish in unexplored regions. To bridge the action space without breaking the continuous-time ODE, we apply a Reward-Guided Mixture Prior directly to the flow’s training targets.

For a batch fraction $\epsilon \in (0, 1)$, we replace a portion of the dataset actions $a_{\text{data}} \in \mathbb{R}^{d_a}$ with Q-maximized synthetic targets. We parameterize an auxiliary Gaussian actor $\pi_\omega : \mathcal{S} \times \mathbb{R}^{d_a} \to \mathcal{P}(\mathbb{R}^{d_a})$ that predicts an absolute action $a$ conditioned on the state and an empirical action. To prevent premature local mode collapse, we optimize $\pi_\omega$ using a Soft Actor-Critic style objective \cite{haarnoja2018soft} with automatic entropy tuning. This dynamically adjusts a dual temperature variable $\alpha \in \mathbb{R}_{>0}$ to maintain the policy's differential entropy near a target of $\mathcal{H}_{\text{target}} = -d_a/2$:
\begin{equation}
\begin{aligned}
 \mathcal{L}_{\text{aug}}(\omega) &= \mathbb{E}_{a \sim \pi_\omega(\cdot | s, a_{\text{data}})} \left[ -Q_\phi(s, a) + \alpha \log \pi_\omega(a | s, a_{\text{data}}) \right], \\
 \mathcal{L}(\alpha) &= \alpha \left[ \mathbb{E}_{a \sim \pi_\omega(\cdot | s, a_{\text{data}})} \left[ \log \pi_\omega(a | s, a_{\text{data}}) \right] - \mathcal{H}_{\text{target}} \right].
\end{aligned}
\end{equation}

During target creation, we suppress sampling variance by extracting the deterministic mode of the trained Gaussian: $a_{\omega} = \text{clip}(\mu_\omega(s, a_{\text{data}}), -1, 1)$. For a detailed discussion on the structural design of this conditional mapping and its impact on the data manifold, see Appendix \ref{sec:appendix_negative_results}. The base flow treats these generated actions as static, stop-gradient ground-truth targets ($x_1$), yielding the mixture target distribution:
\begin{equation}
 \pi_{\text{mix}}(a|s) = (1 - \epsilon) \pi_\beta(a|s) + \epsilon \pi_{\omega}(a|s).
\end{equation}
\paragraph{Density Flattening via Mirror Descent in Continuous Time.}
\label{sec:method_maxent}

To eliminate the popularity bias of the offline dataset, we flatten the target density by integrating a MaxEnt objective directly into the flow model via Functional MD.

To bridge the continuous-time formulation with standard RL notation, we establish a formal equivalence: for any conditional vector field $v_f$, let $\pi^f(\cdot|s)$ denote its induced terminal marginal density, and let $\mathbb{P}_f$ denote its corresponding conditional path measure $\mathbb{P}_f(\cdot \mid X_0, s)$ over the trajectory space $C([0,1], \mathbb{R}^{d_a})$. Consequently, we equate the terminal flow state with the RL action ($x_1 = a$). We define three core components for our optimization: (1) \textit{The Active Policy ($\pi_k, \mathbb{P}_k$)} at MD algorithmic iteration $k$, induced by the fine-tuned vector field $v_k \equiv v_{\theta_{\text{fine}}}^{(k)}$; (2) \textit{The MD Anchor ($\pi_{k-1}, \mathbb{P}_{k-1}$)}, induced by a slowly moving target network $v_{k-1} \equiv v_{\theta_{\text{target}}}^{(k-1)}$ to stabilize Adjoint Matching \citep{li2026qlearning}; and (3) \textit{The Mixture Prior ($\pi_{\text{mix}}, \mathbb{P}_{\text{mix}}$)}, natively induced by the base flow trained against the augmented dataset (Section \ref{sec:method_mixture}), with its static vector field denoted as $v_{\text{mix}} \equiv v_{\theta_{\text{base}}}$.

We introduce state-dependent functions $\beta: \mathcal{S} \to \mathbb{R}_{>0}$ to control policy greediness, and $\eta: \mathcal{S} \to \mathbb{R}_{>0}$ to modulate the density-flattening penalty. To dictate path regularization, we employ a weighting function $\lambda: \mathcal{S} \to (0,1)$ that interpolates between the control cost against the static mixture base flow ($v_{\text{mix}}$) and the slow-moving target anchor ($v_{k-1}$). We consolidate these dual continuous-time path constraints into a convex path regularization term $\mathcal{R}_{\text{path}}(\mathbb{P}_k)$:
\begin{equation}
 \mathcal{R}_{\text{path}}(\mathbb{P}_k) = (1-\lambda(s)) D_{\text{KL}}(\mathbb{P}_k \parallel \mathbb{P}_{k-1}) + \lambda(s) D_{\text{KL}}(\mathbb{P}_k \parallel \mathbb{P}_{\text{mix}}).
\end{equation}
This ensures the total path temperature remains normalized to 1, yielding the unified local fine-tuning objective:
\begin{equation}
 \mathcal{J}_{\text{local}}(\pi_k) = \mathbb{E}_{a \sim \pi_k} \left[ \frac{1}{\beta(s)} Q_\phi(s,a) - \frac{1}{\eta(s)} \log \pi_{k-1}(a|s) \right] - \mathcal{R}_{\text{path}}(\mathbb{P}_k).
 \label{eq:local_objective}
\end{equation}

As established in Equation \ref{eq:girsanov_kl}, these continuous-time KL path constraints map directly to the expected $L_2$ distances between their respective vector fields. Because both costs act as quadratic energetic penalties over the same active variable $v_k$, they can be analytically unified.

\begin{proposition}[Equivalence of Convex KL Path Costs]
\label{prop:dual_kl_equivalence}
For any state-dependent weight $\lambda(s) \in (0,1)$, minimizing the path regularization $\mathcal{R}_{\mathrm{path}}(\mathbb{P}_k)$ with respect to the active drift $v_k$ is exactly equivalent to minimizing a single KL path cost $D_{\mathrm{KL}}(\mathbb{P}_k \parallel \mathbb{P}_{\mathrm{ref}})$ against an interpolated reference vector field $v_{\mathrm{ref}}$, defined as:
\begin{equation}
 v_{\mathrm{ref}}(x_t, t|s) = \lambda(s) v_{\mathrm{mix}}(x_t, t|s) + (1-\lambda(s)) v_{k-1}(x_t, t|s).
 \label{eq:interpolated_v_ref}
\end{equation}
\end{proposition}

The full proof is provided in Appendix \ref{sec:appendix_proof_kl}. Following Proposition \ref{prop:dual_kl_equivalence}, we utilize $v_{\text{ref}}$ directly in our Adjoint solver. 

Because the total effective temperature of this unified path cost is 1, the lean adjoint state must be initialized with the functional gradients of the terminal objective. We embed the Q-gradient and the MD score penalty directly into the terminal boundary condition at $t=1$:
\begin{equation}
 \tilde{y}(x_1, 1) = - \left( \frac{1}{\beta(s)} \nabla_{x_1} Q_\phi(s, x_1) - \frac{1}{\eta(s)} \nabla_{x_1} \log \pi_{k-1}(x_1|s) \right).
 \label{eq:terminal_adjoint_condition}
\end{equation}
By initializing the backward Adjoint ODE with this combined terminal state and regressing the active vector field against $v_{\text{ref}}$, ME-AM natively enforces both the geometric mixture anchor and the MD entropy maximization within a single, coherent optimization step.

\paragraph{Score Estimation and Amortized Implementation.}
\label{sec:method_score}

To bypass the numerical explosion of the analytical score at the terminal boundary ($t \to 1$), we train a lightweight auxiliary network $S_\psi(a, \sigma | s)$ via a variance-weighted denoising objective \cite{ho2020denoising}. By predicting injected noise across a continuous log-uniform spectrum $\sigma \in [\sigma_{\min}, \sigma_{\max}] \subset \mathbb{R}_{>0}$, we maintain capacity across microscopic scales while stabilizing optimization gradients:
\begin{equation}
 \mathcal{L}_{\text{score}}(\psi) = \mathbb{E}_{\sigma \sim \text{LogU}(\sigma_{\min}, \sigma_{\max}), \, a \sim \pi_{k-1}, \, z \sim \mathcal{N}(0, I_{d_a})} \left[ \left\| S_\psi(a + \sigma z, \sigma | s) - z \right\|^2 \right].
 \label{eq:variance_weighted_score}
\end{equation}

Because the base ODE ($v_{\text{mix}}$) inherently resolves global transport, we avoid iterative annealed sampling \cite{song2019generative}. To compute the exact MD penalty for the Adjoint initialization (Equation \ref{eq:terminal_adjoint_condition}), we perform a single deterministic query at the noise floor $\sigma_{\min}$ to algebraically recover the score: $-\nabla_{x_1} \log \pi_{k-1}(x_1|s) \approx S_\psi(x_1, \sigma_{\min} | s) / \sigma_{\min}$. 

Executing a full inner-loop optimization for every MD step is computationally prohibitive. We amortize this process by employing the continuous-time actor-critic architecture utilized in QAM \cite{li2026qlearning}. In addition to the conventional actor and critic networks, and their respective slow-moving targets, maintained in standard Adjoint Matching, our framework concurrently updates the MD target anchor $v_{k-1}$, the auxiliary score network $S_\psi$, and the geometric expansion network $\pi_\omega$. We refer the reader to Appendix \ref{sec:appendix_algorithm} for the complete ME-AM pseudo-code.
\paragraph{Theoretical Guarantee.}
By executing the ME-AM updates, we mathematically achieve the tempered target distribution introduced at the beginning of Section \ref{sec:method}, resolving the SBD. Because we decouple the state-dependent path interpolation weight $\lambda(s)$ from the reward scale $\beta(s)$ and entropy scale $\eta(s)$, ME-AM attains explicit analytical control over the final geometry and greediness of the policy. While we formally define these parameters as state-dependent to ensure our theoretical guarantees hold globally across the MDP, our practical implementation utilizes fixed, state-independent constants to simplify optimization.

\begin{theorem}[Optimal Stationary Policy of ME-AM]
\label{thm:optimal_policy}
Given the local fine-tuning objective $\mathcal{J}_{\mathrm{local}}(\pi_k)$ defined in Equation \ref{eq:local_objective}, assume the amortized updates reach a stationary point such that the sequence of active policies converges to the anchor policy ($\pi_k \to \pi_{k-1} \equiv \pi^*$). The optimal stationary policy $\pi^*(a|s)$ converges exactly to:
\begin{equation}
 \pi^*(a|s) = \frac{1}{Z(s)} \left( \pi_{\mathrm{mix}}(a|s) \right)^{\frac{\lambda(s) \eta(s)}{1 + \lambda(s) \eta(s)}} \exp \left( \frac{\eta(s)}{\beta(s)(1 + \lambda(s) \eta(s))} Q_\phi(s,a) \right),
\end{equation}
where $Z(s)$ is the state-dependent partition function.
\end{theorem}

The full proof is provided in Appendix \ref{sec:appendix_proof_optimal}. Crucially, because $\lambda(s), \eta(s) > 0$, the exponent applied to the base prior, $\delta(s) = \frac{\lambda(s) \eta(s)}{1 + \lambda(s) \eta(s)}$, is bounded between $0$ and $1$ across the entire state space. This formally guarantees that ME-AM flattens the peaked behavior density, exponentially decoupling the agent's exploration from empirical data frequencies and achieving the exact tempered exploration required to bypass popularity bias. Note that Theorem \ref{thm:optimal_policy} assumes access to the exact analytical score. We formally characterize the optimal stationary policy under our practical, Gaussian-smoothed approximation in Appendix \ref{sec:appendix_smoothed_policy}.

\section{Experiments and Results}
\label{sec:experiments}

\input{full_results_agg.tex}

We conduct experiments to evaluate the effectiveness of our method on long-horizon, sparse-reward domains and compare it against a set of representative baselines.

\textbf{Domains and datasets.} We follow an experimental setup similar to that employed by \citet{li2026qlearning}. We focus on two challenging domains from OGBench~\citep{park2025ogbench}: \texttt{puzzle-4x4} (\texttt{p44}) and \texttt{cube-triple} (\texttt{ct}). For \texttt{cube-triple}, we utilize the default \texttt{play} dataset. For \texttt{puzzle-4x4}, we employ the larger 100M-size dataset from \citet{park2025horizon} and adopt the sparse reward definition following \citet{li2025reinforcement}. Both domains require the agent to solve complex tasks from diverse offline behavior data that can only be accurately captured by expressive generative models. Furthermore, following \citet{li2025reinforcement}, we learn action chunking policies with a chunk size of $h=5$. Because action chunking outputs high-dimensional actions exhibiting complex behavior distributions, robust policy extraction becomes critical, making these domains an ideal testbed for ME-AM. Detailed domain specifications are provided in \cref{appendix:domains}.

\textbf{Comparisons.} We benchmark ME-AM against a suite of  SOTA algorithms. Our comparisons include standard Adjoint Matching frameworks (QAM \citep{li2026qlearning}, QAM-E, QAM-F, BAM), action-gradient diffusion models (CGQL, CGQL-L, CGQL-M, DAC \citep{fang2025diffusion}, QSM \citep{qsm_psenka2024}), backpropagation-based flow models (FQL, FBRAC \citep{park2025flow}, FEdit \citep{dong2026expo}), value-weighted regression and post-processing methods (IFQL \citep{hansen2023idql}, FAWAC \citep{park2025flow}, DSRL \citep{wagenmaker2025steering}), and the standard continuous-control baseline ReBRAC \citep{tarasov2024revisiting}. Finally, to evaluate sample efficiency during offline-to-online fine-tuning, we exclusively add RLPD \citep{ball2023efficient}, which is evaluated only in the online phase.

\textbf{Results.} As shown in Table~\ref{tab:agg-sparse-results}, ME-AM significantly outperforms prior  SOTA methods on aggregate performance. To understand the specific mechanisms driving this performance, the remainder of this section is designed to answer the following research questions.

\begin{figure}[t]
  \centering
  % RQ1: Geometric Expansion
  \begin{subfigure}[b]{0.24\textwidth}
    \centering
    \includegraphics[width=\linewidth]{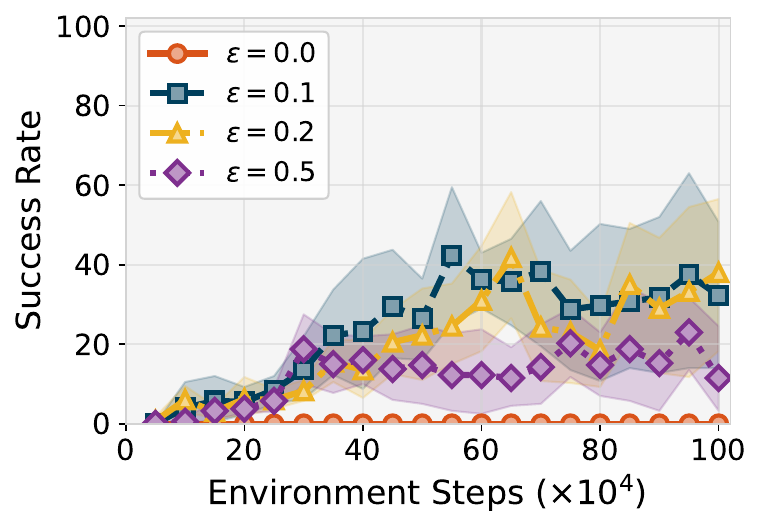}
    \caption{Mixture probability $\epsilon$}
    \label{fig:rq1_ablations}
  \end{subfigure}\hfill
  % RQ2: KL Regularization (contains two side-by-side plots)
  \begin{subfigure}[b]{0.48\textwidth}
    \centering
    \includegraphics[width=0.48\linewidth]{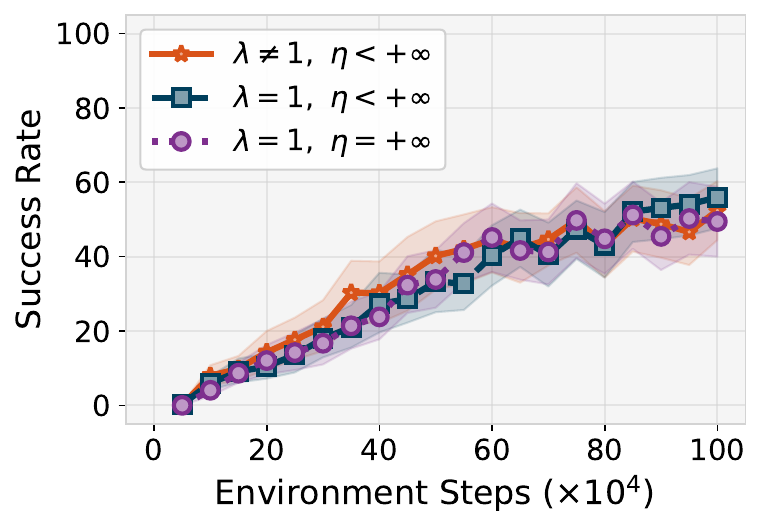}\hfill
    \includegraphics[width=0.48\linewidth]{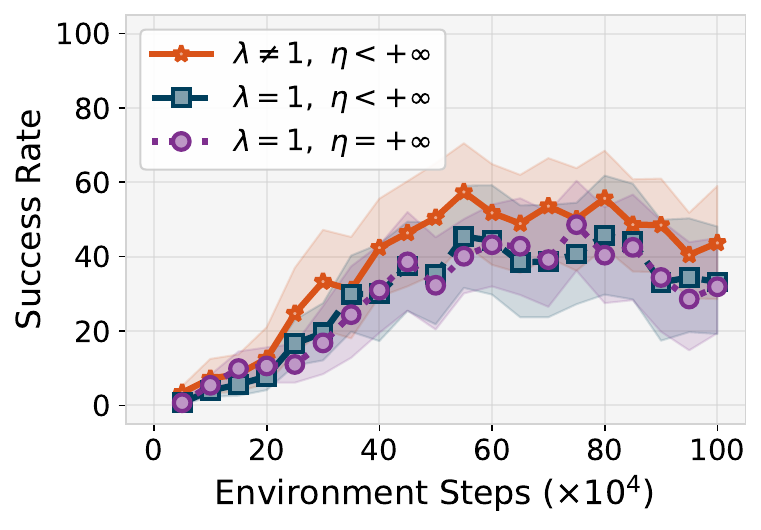}
    \caption{KL path regularization}
    \label{fig:rq2_ablations}
  \end{subfigure}\hfill
  % RQ3: Noise Floor
  \begin{subfigure}[b]{0.24\textwidth}
    \centering
    \includegraphics[width=\linewidth]{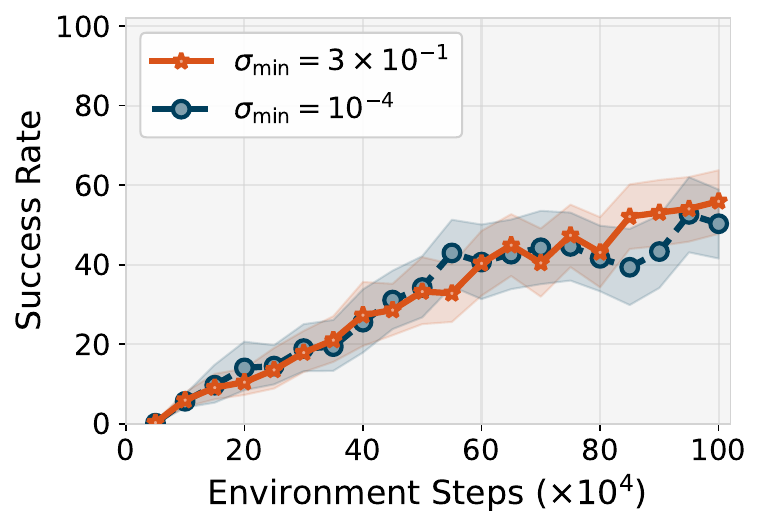}
    \caption{Noise floor $\sigma_{\min}$}
    \label{fig:rq3_ablations}
  \end{subfigure}
  \vspace{-5pt}
  \caption{\textbf{Ablation Studies on \texttt{p44}.} (a) Geometric expansion via Mixture Prior (tasks 2 and 4, 4 seeds). (b) Explicit KL path constraint in standard and noisy settings (5 tasks, 8 seeds). (c) Evaluation noise floor $\sigma_{\min}$ (5 tasks, 8 seeds).}
  \label{fig:all_ablations}
  \vspace{-10pt}
\end{figure}

\textbf{RQ1: How important is the geometric expansion?} As shown in Figure~\ref{fig:all_ablations}a, geometric expansion is essential in the \texttt{p44} domain, where relying on the behavior support ($\epsilon = 0$) causes a total performance collapse. Using moderate expansion values ($\epsilon \in \{0.1, 0.2\}$) safely bridges structural voids and yields clear performance improvements over the QAM-E baseline. However, applying excessively large values ($\epsilon = 0.5$) degrades performance, suggesting that the expanded prior becomes overly unimodal and washes out essential geometric details.

\textbf{RQ2: What is the role of explicit KL path regularization in continuous Mirror Descent?} As shown in Figure~\ref{fig:all_ablations}b (left), the variant enforcing an explicit KL constraint ($\lambda \neq 1, \eta < \infty$) slightly underperforms the pure density-flattening ablation ($\lambda=1, \eta < \infty$). This observation motivated the ablation in Figure~\ref{fig:all_ablations}b (right), where instead of deterministically utilizing the action mean ($\mu_\omega$) during geometric expansion, we inject stochasticity by sampling random actions from $\pi_\omega$. In this noisier setting, the KL constraint proves essential for stabilizing the training dynamics, a result that aligns with recent findings on verifier-constrained entropy maximization \citep{deverifier}.

\textbf{RQ3: How does density flattening interact with the noise floor ($\sigma_{\min}$)?} As shown in Figure~\ref{fig:all_ablations}c for the \texttt{p44} domain, we found that using a larger evaluation noise floor ($\sigma_{\min} = 0.3$) leads to improved performance compared to using a standard small noise level ($10^{-4}$). We believe this suggests that the increased noise floor acts as a spatial smoothing factor, which can be useful for stabilizing optimization in certain tasks. We provide a formal justification of this smoothing phenomenon in Appendix~\ref{sec:appendix_smoothed_policy}.

\begin{figure}[t]

 \centering
 \includegraphics[width=0.8\columnwidth]{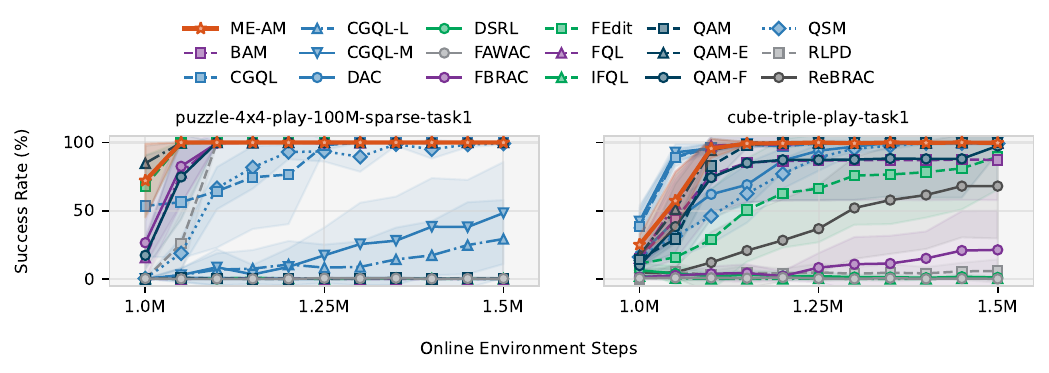}
\caption{\textbf{Offline-to-Online Fine-Tuning Learning Curves (8 seeds).} ME-AM significantly improves upon vanilla QAM.}
 \vspace{-10pt}
 \label{fig:online_finetuning}
\end{figure}

\textbf{RQ4: How effective is our method in offline-to-online fine-tuning?} As shown in Figure~\ref{fig:online_finetuning}, our method significantly improves online performance compared to vanilla QAM, achieving performance on par with QAM-edit, while eliminating the need for an additional edit policy for online exploration.

\section{Conclusion}
\label{sec:conclusion}

In this work, we introduced Maximum Entropy Adjoint Matching (ME-AM) to resolve the Support-Binding dilemma in offline RL by extending the valid action manifold and flattening the empirical behavioral density. While ME-AM yields strong performance on sparse-reward benchmarks and maintains fast inference speeds, the auxiliary score estimation and manifold expansion mechanisms introduce a noticeable training-time computational overhead. To address this limitation, future work could explore dynamic scheduling techniques, applying entropy regularization primarily during later training stages. Furthermore, the framework could be structurally simplified by adopting generative reparameterizations that yield the score natively, or by reframing the adjoint objective with flexible distributional constraints, such as Wasserstein distances, to organically encompass out-of-distribution modes without the need for a separate expansion network.

\bibliographystyle{plainnat}
\bibliography{references}

\newpage
\appendix

\section{Extended Background and Methodological Details}
\label{sec:appendix_extended_background}

Following the discussion in Section \ref{sec:preliminaries} regarding the numerical fragility of the score function near the terminal boundary ($t \to 1$), we detail the algebraic mechanics of this singularity. Let $p^{*, \text{fine}}_t(\cdot|s)$ denote the exact theoretical marginal density of the state-conditioned flow, representing the true analytical target of our approximate density $p^{\text{fine}}_t(\cdot|s)$. Similarly, let $v^{*, \text{fine}} : \mathbb{R}^{d_a} \times [0,1] \times \mathcal{S} \to \mathbb{R}^{d_a}$ denote its corresponding optimal marginal vector field, which the neural network $v_{\theta_{\text{fine}}}$ aims to approximate. Under the independent linear path $X_t = (1-t)X_0 + t X_1$ utilized in our formulation, the base noise $X_0 \sim \mathcal{N}(0, I_{d_a})$ and the target data $X_1$ are uncoupled. As established in the literature \cite{domingo-enrich2025adjoint, holderrieth2025introductionflowmatchingdiffusion, de2025flow}, the analytical nominal score function is exactly given by:
\begin{equation}
 \nabla_{x_t} \log p^{*, \text{fine}}_t(x_t|s) = \frac{t \cdot v^{*, \text{fine}}(x_t, t|s) - x_t}{1-t}.
 \label{eq:score_singularity}
\end{equation}

Analytically, this uncoupled formulation ensures the optimal marginal velocity field at the terminal boundary evaluates to $v^{*, \text{fine}}(x_1, 1|s) = \mathbb{E}[X_1 - X_0 \mid X_1 = x_1] = x_1 - \mathbb{E}[X_0] = x_1$. Therefore, for the exact generative process, the numerator naturally decays at a rate of $\mathcal{O}(1-t)$ as $t \to 1$, offsetting the vanishing denominator and ensuring that Equation \ref{eq:score_singularity} evaluates to a well-defined limit via L'Hôpital's rule.

In practice, replacing the optimal vector field with a parametric neural approximation $v_{\theta_{\text{fine}}}(x_t, t|s)$ introduces an error $e(x_t, t) = v_{\theta_{\text{fine}}}(x_t, t|s) - v^{*, \text{fine}}(x_t, t|s)$. The empirical score estimate becomes:
\begin{equation}
 \nabla_{x_t} \log p^{\text{fine}}_t(x_t|s) \approx \frac{t \cdot \big(v^{*, \text{fine}}(x_t, t|s) + e(x_t, t)\big) - x_t}{1-t} = \nabla_{x_t} \log p^{*, \text{fine}}_t(x_t|s) + \frac{t \cdot e(x_t, t)}{1-t}.
\end{equation}

Unconstrained neural networks do not naturally enforce the $\mathcal{O}(1-t)$ boundary decay required to cancel the denominator \cite{hu2025improving}. Consequently, the residual term $\frac{t \cdot e(x_t, t)}{1-t}$ asymptotically diverges as $t \to 1$. This instability is particularly acute in offline RL: during the initial phase of actor-critic training, the weights $\theta_{\text{fine}}$ are poorly initialized, exacerbating the approximation error and causing the score to explode. Injecting these unbounded gradients into the Adjoint ODE terminal condition ($\tilde{y}(x_1, 1)$) immediately destabilizes the backward pass.

Recent work \cite{hu2025improving} proposes mitigating this specific boundary violation via a subtraction-based reparameterization, which under our settings takes the form: $\tilde{v}_{\theta_{\text{fine}}}(x_t, t|s) = x_t + v_{\theta_{\text{fine}}}(x_t, t|s) - v_{\theta_{\text{fine}}}(x_t, 1|s)$. While this construction explicitly enforces the theoretical terminal condition $\tilde{v}_{\theta_{\text{fine}}}(x_1, 1|s) = x_1$ and stabilizes the analytical limit, it introduces additional computational overhead during inference, as evaluating the constrained vector field requires two distinct forward passes (at $t$ and $t=1$) for every single integration step. This structural limitation necessitates our alternative decoupled score parameterization detailed in Section \ref{sec:method}.

\section{Theoretical Analysis and Proofs}
\label{sec:appendix_theoretical_analysis}
The majority of the mathematical justifications presented in this section are direct results from \cite{domingo-enrich2025adjoint}, provided here to ensure the paper remains self-contained.
\subsection{Equivalence of KL Formulations}
\label{sec:appendix_soc}

We establish the equivalence between the KL divergence formulation used in Eq.~\ref{eq:girsanov_kl} and the control-based expression presented in \cite{domingo-enrich2025adjoint}. The latter expresses the KL divergence between path measures in terms of a quadratic control cost.

We consider a control-affine stochastic differential equation of the form
\begin{equation}
 dX_t = \left( b_{\text{base}}(X_t,t) + g_t\, u(X_t,t) \right) dt + g_t\, dW_t,
\end{equation}
where $X_t \in \mathbb{R}^{d_a}$, $b_{\text{base}} : \mathbb{R}^{d_a} \times [0,1] \to \mathbb{R}^{d_a}$ is the base drift, $u : \mathbb{R}^{d_a} \times [0,1] \to \mathbb{R}^{d_a}$ is a control signal, $g_t = \sqrt{2(1-t)/t}$ is the memoryless noise schedule \cite{domingo-enrich2025adjoint}, and $(W_t)$ is a standard $d_a$-dimensional Brownian motion.

Let $\mathbb{P}_{\text{base}}$ and $\mathbb{P}_{\text{fine}}$ denote the path measures induced by the uncontrolled ($u=0$) and controlled processes, respectively. Under this formulation, the KL divergence between these measures is given by the quadratic control cost \cite{domingo-enrich2025adjoint}:
\begin{equation}
 D_{\mathrm{KL}}(\mathbb{P}_{\text{fine}} \parallel \mathbb{P}_{\text{base}}) = \frac{1}{2} \mathbb{E}_{\boldsymbol{X} \sim \mathbb{P}_{\text{fine}}} \left[ \int_0^1 \|u(X_t,t)\|^2 dt \right].
\end{equation}

In our setting, the dynamics are parameterized via a vector field $v_\theta$. Under the memoryless flow matching construction, the corresponding drift simplifies to
\begin{equation}
 b(X_t,t) = 2 v_\theta(X_t,t) - \frac{1}{t} X_t,
\end{equation}
which induces the base and fine-tuned drifts
\begin{equation}
 b_{\text{base}}(X_t,t) = 2 v_{\theta_{\text{base}}}(X_t,t) - \frac{1}{t} X_t, \quad
 b_{\text{fine}}(X_t,t) = 2 v_{\theta_{\text{fine}}}(X_t,t) - \frac{1}{t} X_t.
\end{equation}

To match the controlled process to the fine-tuned dynamics, we equate the drifts:
\begin{equation}
 b_{\text{base}}(X_t,t) + g_t u(X_t,t) = b_{\text{fine}}(X_t,t).
\end{equation}
Canceling the shared $-\frac{1}{t}X_t$ term yields the control signal
\begin{equation}
 u(X_t,t) = \frac{2}{g_t}\left(v_{\theta_{\text{fine}}}(X_t,t) - v_{\theta_{\text{base}}}(X_t,t)\right).
\end{equation}

Substituting this expression into the quadratic control cost gives
\begin{equation}
 \frac{1}{2} \mathbb{E} \left[ \int_0^1 \left\| \frac{2}{g_t}\left(v_{\theta_{\text{fine}}} - v_{\theta_{\text{base}}}\right) \right\|^2 dt \right]
 = \mathbb{E} \left[ \int_0^1 \frac{2}{g_t^2} \left\| v_{\theta_{\text{fine}}} - v_{\theta_{\text{base}}} \right\|^2 dt \right],
\end{equation}
which recovers the KL expression used in Eq.~\ref{eq:girsanov_kl}. This establishes the equivalence between the control-based formulation of \cite{domingo-enrich2025adjoint} and the vector-field-based formulation used in our method.

\subsection{Proof of Proposition \ref{prop:dual_kl_equivalence}: Equivalence of Convex KL Path Costs}
\label{sec:appendix_proof_kl}

\textit{Notation Convention:} For the subsequent proofs, we adhere to the notational conventions established in Section \ref{sec:method_maxent}. Specifically, the terminal flow state is equated with the RL action ($x_1 \equiv a$), and any conditional vector field $v_f$ implicitly defines a terminal policy $\pi^f$ and a conditional path measure $\mathbb{P}_f$. Furthermore, the parameters $\beta(s)$, $\eta(s)$, and $\lambda(s)$ are treated as state-dependent functions to ensure global optimality across the MDP.

As stated in Section \ref{sec:method_maxent}, the continuous-time SOC framework allows us to analytically unify the convex combination of two KL path constraints into a single regression target. Here, we provide the formal algebraic proof.

\begin{proof}
Let the path regularization $\mathcal{R}_{\mathrm{path}}(\mathbb{P}_k)$ from Equation \ref{eq:local_objective} be defined by the convex combination of the two continuous-time KL divergences. As established in Equation \ref{eq:girsanov_kl} (and detailed in Appendix \ref{sec:appendix_soc}), the KL divergence between path measures maps exactly to the expected $L_2$ distance between their respective vector fields scaled by $2/g_t^2$.

Let $\mathcal{L}_{\mathrm{penalty}}(v_k)$ denote the instantaneous integrand of this combined path cost at time $t$:
\begin{equation}
 \mathcal{L}_{\mathrm{penalty}}(v_k) = \frac{2}{g_t^2} \left[ (1-\lambda(s)) \|v_k - v_{k-1}\|^2 + \lambda(s) \|v_k - v_{\mathrm{mix}}\|^2 \right].
\end{equation}

Factoring out the shared $2/g_t^2$ multiplier, we expand the quadratic terms:
\begin{align}
 \frac{g_t^2}{2} \mathcal{L}_{\mathrm{penalty}}(v_k) &= (1-\lambda(s)) \left( \|v_k\|^2 - 2 v_k^\top v_{k-1} + \|v_{k-1}\|^2 \right) \nonumber \\
 &\quad + \lambda(s) \left( \|v_k\|^2 - 2 v_k^\top v_{\mathrm{mix}} + \|v_{\mathrm{mix}}\|^2 \right).
\end{align}

Grouping the terms by the active variable $v_k$, and noting that $(1-\lambda(s)) + \lambda(s) = 1$, we obtain:
\begin{equation}
 = \|v_k\|^2 - 2 v_k^\top \left( \lambda(s) v_{\mathrm{mix}} + (1-\lambda(s)) v_{k-1} \right) + \lambda(s) \|v_{\mathrm{mix}}\|^2 + (1-\lambda(s)) \|v_{k-1}\|^2.
\end{equation}

Substituting the definition of the interpolated reference flow, $v_{\mathrm{ref}} = \lambda(s) v_{\mathrm{mix}} + (1-\lambda(s)) v_{k-1}$ (Equation \ref{eq:interpolated_v_ref}), and completing the square yields:
\begin{equation}
 = \|v_k\|^2 - 2 v_k^\top v_{\mathrm{ref}} + C' = \|v_k - v_{\mathrm{ref}}\|^2 + C.
\end{equation}

Here, the residual $C = C' - \|v_{\mathrm{ref}}\|^2$ depends exclusively on the fixed anchors ($v_{\mathrm{mix}}$ and $v_{k-1}$) and the state-dependent hyperparameter $\lambda(s)$. Multiplying the scale factor back in, the total cost simplifies to:
\begin{equation}
 \mathcal{L}_{\mathrm{penalty}}(v_k) = \frac{2}{g_t^2} \|v_k - v_{\mathrm{ref}}\|^2 + \frac{2}{g_t^2} C.
\end{equation}

Because the functional derivative of the residual with respect to the active policy is exactly zero ($\frac{\delta C}{\delta v_k} = 0$), optimizing $v_k$ against the two separate convex anchors is mathematically identical to solving a single optimal control problem against the interpolated flow $v_{\mathrm{ref}}$ with a total path temperature of 1.
\end{proof}

\subsection{Proof of Theorem \ref{thm:optimal_policy}: Optimal Stationary Policy}
\label{sec:appendix_proof_optimal}

\begin{proof}
We begin by expanding the KL divergences in the local objective (Equation \ref{eq:local_objective}) and incorporating the strict density integration constraint via a Lagrange multiplier $\nu$:
\begin{align}
 \mathcal{L}(\pi_k, \nu) &= \int_{\mathcal{A}} \pi_k(a|s) \left[ \frac{1}{\beta(s)} Q_\phi(s,a) - \frac{1}{\eta(s)} \log \pi_{k-1}(a|s) \right. \nonumber \\
 &\quad \left. - (1-\lambda(s)) \left( \log \pi_k(a|s) - \log \pi_{k-1}(a|s) \right) \right. \nonumber \\
 &\quad \left. - \lambda(s) \left( \log \pi_k(a|s) - \log \pi_{\mathrm{mix}}(a|s) \right) \right] \mathrm{d}a + \nu \left( 1 - \int_{\mathcal{A}} \pi_k(a|s) \mathrm{d}a \right).
\end{align}

Taking the functional derivative with respect to the density $\pi_k(a|s)$ and setting it to zero yields the optimality condition for iteration $k$. Noting that the derivative of the terms $-(1-\lambda(s))\pi_k \log \pi_k - \lambda(s) \pi_k \log \pi_k$ collapses perfectly to $-(1 + \log \pi_k)$, we obtain:
\begin{equation}
 \begin{split}
 \frac{1}{\beta(s)} Q_\phi(s,a) &- \frac{1}{\eta(s)} \log \pi_{k-1}(a|s) - \log \pi_k(a|s) - 1 \\
 &+ (1-\lambda(s)) \log \pi_{k-1}(a|s) + \lambda(s) \log \pi_{\mathrm{mix}}(a|s) - \nu = 0.
 \end{split}
\end{equation}

Rearranging to isolate the active variable $\log \pi_k(a|s)$ gives:
\begin{equation}
 \log \pi_k(a|s) = \frac{1}{\beta(s)} Q_\phi(s,a) - \left( \frac{1}{\eta(s)} - 1 + \lambda(s) \right) \log \pi_{k-1}(a|s) + \lambda(s) \log \pi_{\mathrm{mix}}(a|s) - (1 + \nu).
\end{equation}

At the stationary point of the amortized MD procedure, the sequence of active policies converges to the anchor policy such that $\pi_k \to \pi_{k-1} \equiv \pi^*$. Letting the constant $C = 1 + \nu$, we substitute $\pi^*$ into the optimality condition and group the $\log \pi^*(a|s)$ terms on the left side:
\begin{equation}
 \begin{split}
 \left( 1 + \frac{1}{\eta(s)} - 1 + \lambda(s) \right) \log \pi^*(a|s) &= \left( \frac{1}{\eta(s)} + \lambda(s) \right) \log \pi^*(a|s) \\
 &= \frac{1}{\beta(s)} Q_\phi(s,a) + \lambda(s) \log \pi_{\mathrm{mix}}(a|s) - C.
 \end{split}
\end{equation}

Finally, multiplying the entire equation by $\frac{\eta(s)}{1 + \lambda(s) \eta(s)}$ and exponentiating yields the optimal stationary distribution:
\begin{equation}
 \pi^*(a|s) = \frac{1}{Z(s)} \left( \pi_{\mathrm{mix}}(a|s) \right)^{\frac{\lambda(s) \eta(s)}{1 + \lambda(s) \eta(s)}} \exp \left( \frac{\eta(s)}{\beta(s)(1 + \lambda(s) \eta(s))} Q_\phi(s,a) \right),
\end{equation}
where $Z(s) = \exp \left( \frac{C \eta(s)}{1 + \lambda(s) \eta(s)} \right)$ is the partition function ensuring the density integrates to 1.
\end{proof}

\textit{Remark.} The closed-form stationary policy isolates the mechanical impact of the local objective's hyperparameters. The base prior is flattened by the exponent $\delta(s) = \frac{\lambda(s) \eta(s)}{1 + \lambda(s) \eta(s)} \in (0, 1)$, while the resulting policy greediness is dictated by an effective reward temperature $\kappa(s) = \frac{\beta(s)(1 + \lambda(s) \eta(s))}{\eta(s)}$. Additionally, the path weight $(1-\lambda(s))$ acts as a continuous-time trust region, bounding the update step from the anchor $\pi_{k-1}$ to stabilize iterative fine-tuning against value overestimation.

\subsection{Optimal Policy with Smoothed Entropy Maximization}
\label{sec:appendix_smoothed_policy}

As noted in Section \ref{sec:method_score}, our variance-weighted score estimation queries the auxiliary network at a positive noise floor $\sigma_{\min}$. This effectively estimates the score of a Gaussian-smoothed Kernel Density Estimate (KDE) \citet{vincent2011connection}. Here, we formalize the optimal stationary policy under this smoothed objective.

\begin{proposition}[Optimal Policy with Smoothed Entropy]
\label{prop:smoothed_entropy}
Given the local fine-tuning objective $\mathcal{J}_{\mathrm{local}}(\pi_k)$, suppose we replace the exact target density $\pi_{k-1}$ with its Gaussian-smoothed KDE, $\pi^\sigma_{k-1}(a|s) = (\pi_{k-1} * \mathcal{N}(0, \sigma_{\min}^2 I_{d_a}))(a|s)$. Assume the sequence of active policies converges to the anchor policy ($\pi_k \to \pi_{k-1} \equiv \pi^*$). The optimal stationary policy $\pi^*(a|s)$ satisfies the implicit distribution:
\begin{equation}
 \pi^*(a|s) = \frac{1}{Z_\sigma(s)} \pi_{\mathrm{mix}}(a|s) \exp \left( \frac{1}{\lambda(s) \beta(s)} Q_\phi(s,a) \right) \left[ \pi^{\sigma, *}(a|s) \right]^{-\frac{1}{\lambda(s) \eta(s)}},
\end{equation}
where $\pi^{\sigma, *}$ is the Gaussian-smoothed stationary policy, and $Z_\sigma(s)$ is the state-dependent partition function.
\end{proposition}

\begin{proof}
Following the variational derivation from Appendix \ref{sec:appendix_proof_optimal}, replacing the exact entropy penalty with the smoothed approximation modifies the optimality condition. Setting the functional derivative of the modified Lagrangian to zero and rearranging for the active variable $\log \pi_k(a|s)$ yields:
\begin{equation}
 \log \pi_k(a|s) - (1-\lambda(s)) \log \pi_{k-1}(a|s) = \frac{1}{\beta(s)} Q_\phi(s,a) - \frac{1}{\eta(s)} \log \pi^\sigma_{k-1}(a|s) + \lambda(s) \log \pi_{\mathrm{mix}}(a|s) - C,
\end{equation}
where $C$ incorporates the Lagrange multiplier for the integration constraint. 

At the stationary point, the sequence of active policies converges to the anchor ($\pi_k \to \pi_{k-1} \equiv \pi^*$), and consequently, the smoothed density converges to the smoothed stationary density ($\pi^\sigma_{k-1} \to \pi^{\sigma, *}$). Substituting these limits into the optimality condition yields:
\begin{equation}
 \lambda(s) \log \pi^*(a|s) = \frac{1}{\beta(s)} Q_\phi(s,a) - \frac{1}{\eta(s)} \log \pi^{\sigma, *}(a|s) + \lambda(s) \log \pi_{\mathrm{mix}}(a|s) - C.
\end{equation}

Dividing the entire equation by $\lambda(s)$ and exponentiating isolates $\pi^*(a|s)$, yielding the implicit optimal stationary distribution.
\end{proof}

\textit{Remark.} Exact score matching on finite empirical data typically yields unbounded gradients as $\sigma \to 0$. By explicitly targeting the Gaussian-smoothed policy $\pi^\sigma_{k-1}$, the resulting stationary distribution incorporates the inverse term $\left[ \pi^{\sigma, *}(a|s) \right]^{-1/\lambda(s) \eta(s)}$. This establishes an inverse dependence on the smoothed spatial density, providing a regularizing effect that reduces the policy's sensitivity to isolated empirical data points.

\section{ME-AM Pseudocode}
\label{sec:appendix_algorithm}

In this section, we provide the streamlined pseudocode for ME-AM, summarizing the core training loop and the integration of our geometric expansion and density flattening mechanisms.

\begin{algorithm}[H]
 \caption{Maximum Entropy Adjoint Matching (ME-AM)}
 \label{alg:me_am}
\begin{algorithmic}[1]
 \REQUIRE Dataset $\mathcal{D}$, Mixture fraction $\epsilon$, Interpolation weight $\lambda$, Scales $\beta, \eta$, Noise floor $\sigma_{\min}$, Target network update rate $\rho \in (0,1)$.
 \STATE \textbf{Initialize:} Critic $Q_\phi$, Base flow $v_{\theta_{\text{base}}}$, Fine-tuned flow $v_{\theta_{\text{fine}}}$, Target flow $v_{\theta_{\text{target}}}$, Expansion actor $\pi_\omega$, Score network $S_\psi$.
 \FOR{each training iteration}
  \STATE \textbf{1. Critic \& Expansion Updates:}
  \STATE Update critic $Q_\phi$ via standard TD-learning.
  \STATE Update expansion actor $\pi_\omega$ via SAC-style objective to maximize $Q_\phi(s, a)$.
  
  \STATE \textbf{2. Mixture Base Flow Update (Geometric Expansion):}
  \STATE Sample $(s, a_{\text{data}}) \sim \mathcal{D}$ and $u \sim \mathcal{U}(0,1)$.
  \STATE Target $x_1 \leftarrow \mu_\omega(s, a_{\text{data}})$ if $u < \epsilon$ else $a_{\text{data}}$.
  \STATE Update $\theta_{\text{base}}$ via Conditional Flow Matching (\Cref{CFM_loss}) towards $x_1$.
  
  \STATE \textbf{3. Auxiliary Score Update (Density Flattening):}
  \STATE Sample terminal actions $a$ using target flow $v_{\theta_{\text{target}}}$.
  \STATE Update score network $S_\psi$ via variance-weighted denoising (\Cref{eq:variance_weighted_score}).
  
  \STATE \textbf{4. Adjoint Matching Update (Unified Objective):}
  \STATE Interpolate reference flow: $v_{\text{ref}} \coloneqq \lambda v_{\theta_{\text{base}}} + (1-\lambda) v_{\theta_{\text{target}}}$.
  \STATE Simulate forward trajectory $x_1 \leftarrow \text{ODESolve}(x_0, v_{\text{ref}}, t \in [0,1])$.
  \STATE Compute MD boundary condition: $\tilde{y}(x_1, 1) \leftarrow - \frac{1}{\beta} \nabla_{x_1} Q_\phi(s, x_1) + \frac{1}{\eta \sigma_{\min}} S_\psi(x_1, \sigma_{\min} \mid s)$.
  \STATE Solve Adjoint ODE backward: $\{\tilde{y}(x_t, t)\}_{t=0}^1 \leftarrow \text{ODESolve}(\tilde{y}(x_1, 1), \text{Eq. \ref{eq:lean_adjoint_ode}}, t \in [1,0])$.
  \STATE Update $\theta_{\text{fine}}$ via Adjoint Matching regression (\Cref{eq:am_loss}) against targets $\tilde{y}(x_t, t)$.
  
  \STATE \textbf{5. Target Update:}
  \STATE $\theta_{\text{target}} \leftarrow \rho \theta_{\text{target}} + (1-\rho) \theta_{\text{fine}}$ (and similarly update target critic).
 \ENDFOR
\end{algorithmic}
\end{algorithm}

\section{Experimental Details}
\label{appendix:experimental_details}

In this section, we provide the detailed settings and configurations utilized in our experimental evaluations.

\subsection{Domain Details}
\label{appendix:domains}

We evaluate our method across two challenging domains from OGBench \citep{park2025ogbench}. These specific tasks were selected to cover offline RL challenges characterized by precise manipulation and severe reward sparsity. The dataset size, episode length, and action dimension for each domain are detailed in Table \ref{tab:metadata}. 

All experiments for ME-AM are implemented in JAX \citep{jax2018github} and executed on a high-performance computing cluster utilizing NVIDIA A100 GPUs. For each ME-AM evaluation, we run 8 independent seeds. All plots and tables report the mean success rates with 95\% confidence intervals computed via bootstrapping.

\begin{table}[ht]
  \centering
  \begin{tabular}{@{}cccc@{}}
   \toprule
   \textbf{Tasks} & \textbf{Dataset Size} & \textbf{Episode Length} & \textbf{Action Dimension ($A$)} \\
   \midrule
   \texttt{cube-triple-*} & $3$M & $1000$ & $5$\\
   \texttt{puzzle-4x4-100M-sparse-*} & $100$M & $500$ & $5$\\
  \bottomrule
  \end{tabular}
  \vspace{1mm}
  \caption{\textbf{Domain metadata.} Summary of the selected OGBench domains used for evaluation.}
  \label{tab:metadata}
\end{table}

\begin{figure}[ht]
  \centering
  \includegraphics[width=0.48\linewidth]{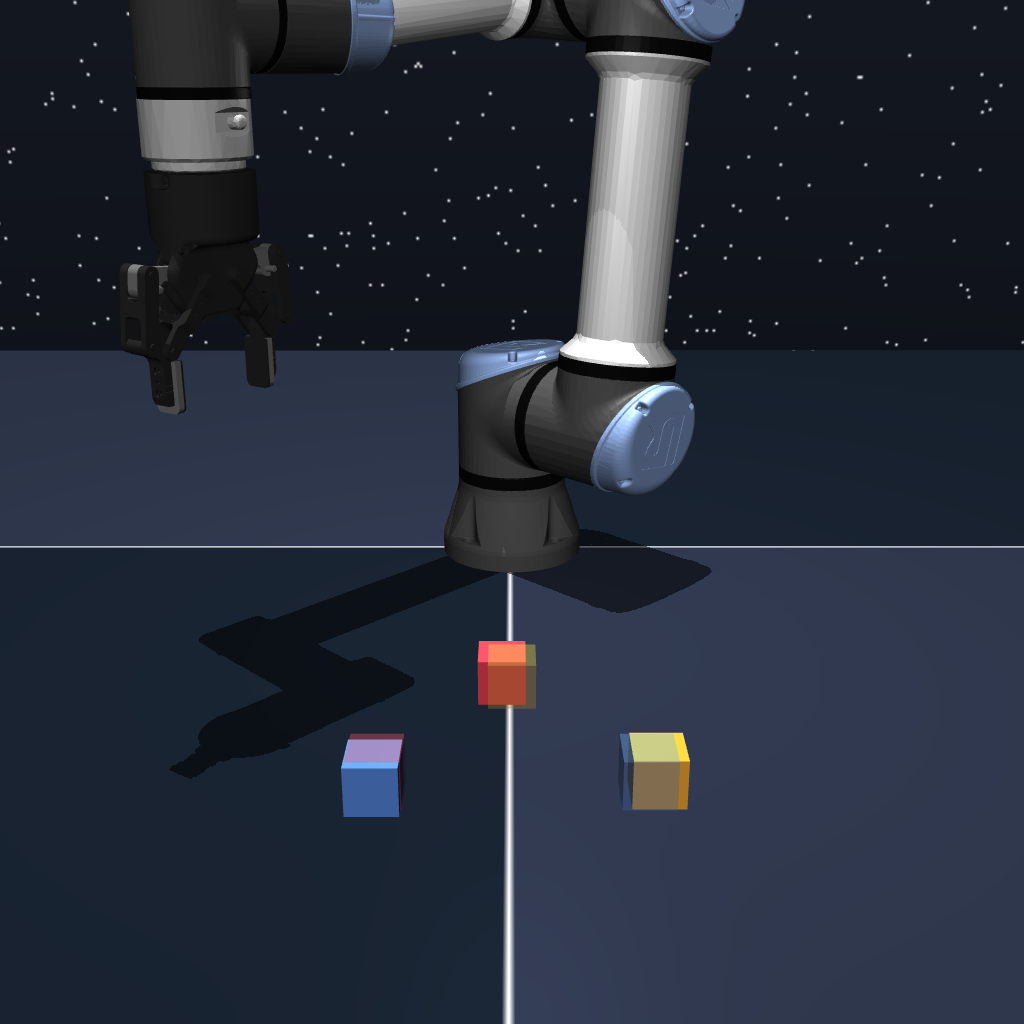}\hfill
  \includegraphics[width=0.48\linewidth]{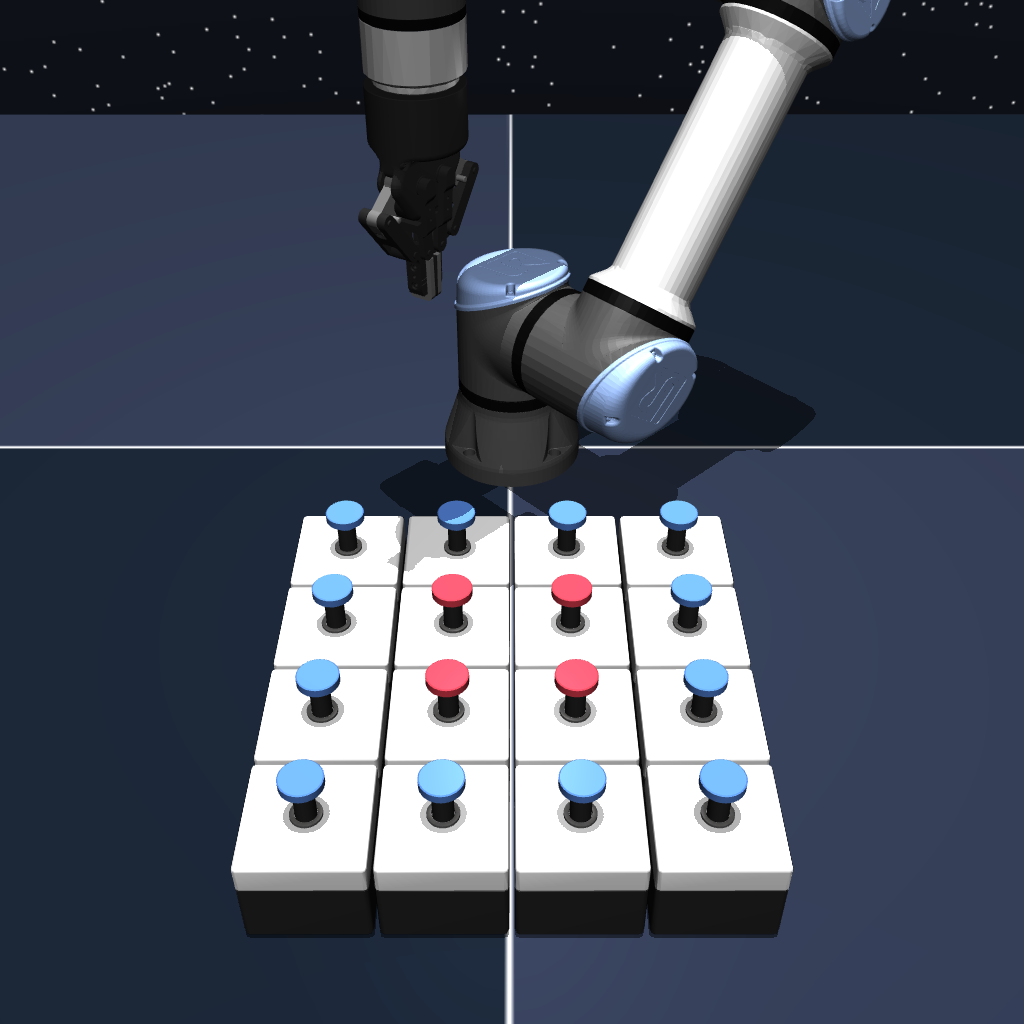} 
  \caption{\textbf{OGBench domains \citep{park2025ogbench}}. \textbf{Left:} \texttt{cube-triple} requires a robotic arm to manipulate 3 cubes from an initial arrangement to a goal arrangement. \textbf{Right:} \texttt{puzzle-4x4-100M-sparse} features a challenging, sparse-reward multi-stage navigation puzzle. The offline datasets for these domains contain multi-modal behaviors that necessitate expressive generative models.}
  \label{fig:env_renders}
\end{figure}

\subsection{Baselines}
\label{appendix:baselines}

While Section \ref{sec:experiments} briefly enumerates our baselines, we provide a more detailed mechanical categorization here to contextualize the comparisons. Following the taxonomy established by \citet{li2026qlearning}, we group the  SOTA offline RL algorithms into several distinct methodological paradigms:

\begin{itemize}
  \item \textbf{Backpropagation-based Flows:} Methods that directly differentiate through the generative process. This includes FBRAC (which backpropagates through the flow policy's multi-step denoising), FQL (which backpropagates through a 1-step distilled policy) \citep{park2025flow}, and BAM. BAM acts as a direct ablation of standard Adjoint Matching, utilizing a basic regression objective whose gradient is mathematically equivalent to backpropagating through a memory-less SDE \citep{li2026qlearning}.
  \item \textbf{Advantage-Weighted Regression:} FAWAC \citep{park2025flow}, which extends the classic AWAC \citep{nair2020awac} framework by modeling the policy with a continuousizing flow model.
  \item \textbf{Critic Gradient Guidance:} Methods inspired by classifier guidance \citep{dhariwal2021diffusion} that leverage the action gradient ($\nabla_a Q(s, a)$) to steer the generative model. This includes DAC \citep{fang2025diffusion}, QSM \citep{qsm_psenka2024}, and the CGQL family (CGQL, CGQL-MSE, CGQL-Linex).
  \item \textbf{Post-processing and Edits:} Approaches that alter the action after the base generative step. This includes DSRL \citep{wagenmaker2025steering}, FEdit (which appends a learnable Gaussian edit policy to the flow \citep{dong2026expo}), and IFQL (which relies on best-of-$N$ rejection sampling for policy extraction \citep{hansen2023idql}).
  \item \textbf{Adjoint Matching:} The QAM framework and its primary variants (QAM-E and QAM-F) \citep{li2026qlearning}, against which ME-AM serves as a direct, unified continuous-time alternative.
  \item \textbf{Gaussian Policies (ReBRAC, RLPD):} Strong, regularized non-generative baselines. We specifically include RLPD \citep{ball2023efficient} exclusively to evaluate sample efficiency during the offline-to-online fine-tuning phase. Because RLPD does not employ a behavioral constraint, \citet{li2026qlearning} trained it from scratch online using a 50/50 sampling ratio (i.e., half of the training batch is sampled from the static offline dataset, and half from the active online replay buffer).
\end{itemize}

\paragraph{Data Sourcing and Seed Matching.}
For the empirical results reported in this paper, we utilize the raw experimental logs published by the authors of QAM, available in their official repository (\url{https://github.com/ColinQiyangLi/qam}). In their original evaluation protocol, the authors ran 12 total seeds per method, allocating 4 seeds exclusively for hyperparameter tuning and 8 new seeds for final evaluation. Because our ME-AM evaluations are conducted over exactly 8 independent seeds, we ensure perfect statistical parity by aggregating only the results from those 8 designated evaluation seeds.

\subsection{Hyperparameter Tuning Protocol}
\label{appendix:hyperparameters}

To ensure fairness with the established baselines, we employ an identical hyperparameter tuning protocol as described by \citet{li2026qlearning}. Specifically, we tune over 4 seeds on 2 representative tasks from each domain (task2 and task4). We maintain all shared hyperparameters across methods (e.g., learning rate, optimizer, batch size, target update rate) exactly equal to those used in the baseline algorithms.

Because ME-AM introduces additional domain-specific hyperparameters ($\lambda$, $\sigma_{\min}$, $\sigma_{\max}$, $\eta$, $\beta$, and $\epsilon$), we deliberately restricted the tuning grid to be as small as possible to avoid overfitting. For the noise floor $\sigma_{\min}$, we tested two discrete values: $0.3$ and $10^{-4}$. Because the norm of the estimated score $S_\psi$ scales inversely with $\sigma_{\min}$, we evaluated corresponding pairs of the inverse entropy scale ($1/\eta$): when testing $\sigma_{\min} = 0.3$, we swept $1/\eta \in \{1.0, 0.1\}$; when testing $\sigma_{\min} = 10^{-4}$, we swept $1/\eta \in \{10^{-3}, 10^{-4}\}$. 

For the path interpolation weight $\lambda$, we tested values in $\{0.95, 0.9, 0.8\}$. We empirically observed that using values smaller than $0.8$ often led to exploding gradients at the beginning of training on certain seeds. For the mixture probability $\epsilon$, we evaluated values in $\{0.0, 0.1, 0.2, 0.5\}$. Finally, we fixed the maximum noise level $\sigma_{\max} = 0.7$ and the inverse reward temperature $1/\beta = 5.0$ globally across all experiments, having initially tested $1/\beta \in \{10.0, 5.0\}$.

Finally, the offline-to-online fine-tuning results were obtained by initializing the agent with the best-performing offline models and subsequently letting them interact directly with the online environment.

Table \ref{tab:rl-hyperparams} summarizes both the shared architectural settings and the final domain-specific hyperparameters used for our main evaluations.

\begin{table}[ht]
  \centering
  \begin{tabular}{@{}lc@{}}
   \toprule
   \textbf{Parameter} & \textbf{Value} \\
   \midrule
   \multicolumn{2}{c}{\textbf{Shared Hyperparameters}} \\
   \midrule
   Batch size & $256$ \\
   Discount factor ($\gamma$) & $0.99$ \\
   Optimizer & Adam \\
   Learning rate & $3 \times 10^{-4}$ \\
   Target network update rate ($\rho$)  & $5 \times 10^{-3}$ \\
   Critic ensemble size & $10$ \\
   Critic target pessimistic coefficient  & $0.5$ \\ 
   UTD ratio & $1$ \\
   Number of flow integration steps  & $10$ \\
   Number of offline training steps & $10^6$ \\
   Number of online environment steps & $0.5 \times 10^6$ \\
   Network architecture (All networks used) & $4$ hidden layers, $512$ width \\
   Optimizer gradient max norm clipping & $1$ \\
   Maximum noise level ($\sigma_{\max}$) & $0.7$ \\
   Inverse reward temperature ($1/\beta$) & $5.0$ \\
   \midrule
   \multicolumn{2}{c}{\textbf{Domain-Specific Hyperparameters: \texttt{p44}}} \\
   \midrule
   Path interpolation weight ($\lambda$) & $1.0$ \\
   Inverse entropy scale ($1/\eta$) & $1.0$ \\
   Noise floor ($\sigma_{\min}$) & $0.3$ \\
   Mixture probability ($\epsilon$) & $0.2$ \\
   \midrule
   \multicolumn{2}{c}{\textbf{Domain-Specific Hyperparameters: \texttt{ct}}} \\
   \midrule
   Path interpolation weight ($\lambda$) & $0.8$ \\
   Inverse entropy scale ($1/\eta$) & $10^{-3}$ \\
   Noise floor ($\sigma_{\min}$) & $10^{-4}$ \\
   Mixture probability ($\epsilon$) & $0.0$ \\
  \bottomrule
  \end{tabular}
  \vspace{1mm}
  \caption{\textbf{ME-AM Hyperparameters.} Global shared settings and final domain-specific values utilized for both the offline and online fine-tuning phases.}
  \label{tab:rl-hyperparams}
\end{table}

\begin{table}[h]
\centering

\resizebox{\textwidth}{!}{%
\begin{tabular}{lccc}
\toprule
\textbf{Method} & \textbf{\begin{tabular}[c]{@{}c@{}}Training \\ Parameters\end{tabular}} & \textbf{\begin{tabular}[c]{@{}c@{}}Inference Speed\\ (ms/step)\end{tabular}} & \textbf{\begin{tabular}[c]{@{}c@{}}Training Speed\\ (ms/step)\end{tabular}} \\ 
\midrule
RLPD & 8,931,867 & 0.02 & 2.52 \\
ReBRAC & 8,927,762 & 0.02 & 2.62 \\
FQL & 9,743,898 & 0.02 & 3.39 \\
FAWAC & 9,740,307 & 0.14 & 3.29 \\
FEdit & 9,748,003 & 0.17 & 3.26 \\
FBRAC & 8,932,370 & 0.14 & 4.33 \\
DSRL & 18,676,277 & 0.16 & 4.64 \\
IFQL & 9,740,307 & 0.56 & 2.15 \\
QSM & 10,015,282 & 0.18 & 3.82 \\
DAC & 10,015,282 & 0.17 & 3.82 \\
BAM & 9,744,410 & 0.14 & 4.96 \\
CGQL & 8,937,490 & 1.21 & 7.77 \\
CGQL-M & 8,937,490 & 1.21 & 8.98 \\
CGQL-L & 8,937,490 & 1.21 & 8.97 \\
\midrule
QAM & 9,744,410 & 0.14 & 5.04 \\
QAM-E & 10,560,043 & 0.16 & 6.31 \\
QAM-F & 10,555,938 & 0.03 & 5.85 \\
\midrule
ME-AM ($\lambda=1.0$) & 11,408,435 & 0.14 & 6.79 \\
ME-AM ($\epsilon=0.0$) & 10,592,802 & 0.14 & 7.28 \\
ME-AM (Full) & 11,408,435 & 0.14 & 7.73 \\
\bottomrule
\end{tabular}%
}
\caption{Computational profiling of ME-AM compared to baseline methods.}
\label{tab:computation_profiling}
\end{table}

\subsection{Ablation Study Details}
\label{appendix:ablation_details}

In Section \ref{sec:experiments}, we presented a series of ablation studies to isolate the impact of our core algorithmic components. 

\paragraph{Geometric Expansion (RQ1).} 
To evaluate the impact of the Reward-Guided Mixture Prior (Figure \ref{fig:all_ablations}a), we fixed all shared and domain-specific hyperparameters to their optimal \texttt{p44} values as detailed in Table \ref{tab:rl-hyperparams}. To isolate the effect of the geometric expansion, we exclusively varied the mixture probability $\epsilon \in \{0.0, 0.1, 0.2, 0.5\}$. The $\epsilon=0.0$ setting directly corresponds to relying on the empirical behavior support, whereas $\epsilon > 0$ introduces the expanded target distribution.

\paragraph{KL Path Regularization (RQ2, Left).} 
To investigate the necessity of the explicit KL path constraint in the deterministic setting (Figure \ref{fig:all_ablations}b, left), we again maintained the optimal baseline hyperparameters and evaluated three distinct optimization scenarios. For clarity, we map the mathematical parameterizations used in our code directly to the conceptual formulations discussed in the main text:
\begin{itemize}
    \item \textbf{Standard Adjoint Matching (No MD):} We set $\lambda = 1.0$ and $1/\eta = 0.0$. This completely removes the Mirror Descent density flattening penalty, recovering standard AM behavior.
    \item \textbf{Pure Density Flattening ($\lambda = 1.0, \eta < \infty$):} We set $\lambda = 1.0$ and $1/\eta = 1.0$. This applies the active entropy maximization but ties the path constraint entirely to the static mixture base flow, applying no constraint against the active, slow-moving target anchor.
    \item \textbf{Explicit KL Regularization ($\lambda \neq 1.0, \eta < \infty$):} We set $\lambda = 0.95$ and $1/\eta = 1.0$. This explicitly enforces the continuous-time KL path constraint by pulling the optimization toward the active target anchor $v_{k-1}$.
\end{itemize}

\paragraph{KL Path Regularization in Noisy Settings (RQ2, Right).}
To better understand the conditions under which explicit KL path regularization becomes critical, we simulated a high-variance optimization landscape (Figure \ref{fig:all_ablations}b, right). Instead of using the deterministic mean of the auxiliary actor to expand the manifold, we introduced stochasticity by sampling actions directly from the geometric expansion policy ($a_{\text{aug}} \sim \pi_\omega$). 

In this noisy regime, undirected sampling frequently pushes the model into regions with weak or negative reward signals, causing the initial Q-gradient to vanish and triggering critic collapse. To bootstrap a meaningful gradient, we had to aggressively scale the reward signal by setting $1/\beta = 15.0$ and increase the mixture probability to $\epsilon = 0.3$. However, while this extreme scaling prevents early collapse, it rapidly induces severe OOD overestimation later in training, where Q-values artificially inflate while actual evaluation success rates plummet. To mitigate this OOD explosion, we   clipped the terminal Q-gradient during the Adjoint backward step to a maximum norm of 1.

Under these noisy conditions, we evaluated the stabilizing effect of the path constraint. We found that heavily prioritizing the slow-moving target anchor, by setting $\lambda = 0.2$ alongside a moderate density penalty $1/\eta = 0.1$, was necessary to stabilize the training dynamics. Notably, the engineering manipulations applied here (specifically terminal gradient clipping) were strictly isolated to this specific ablation study to stress-test the KL constraint; they are not utilized in our standard ME-AM algorithm.

\paragraph{Noise Floor (RQ3).} 
To evaluate the effect of the evaluation noise floor (Figure \ref{fig:all_ablations}c), we fixed the path interpolation weight to $\lambda = 1.0$. We then compared the performance of a larger noise floor ($\sigma_{\min} = 0.3$ paired with $1/\eta = 1.0$) against a standard minimal noise level ($\sigma_{\min} = 10^{-4}$ paired with $1/\eta = 10^{-3}$).

\subsection{Computational Complexity and Deployment Trade-offs.} 
We benchmark the computational requirements of ME-AM against QAM and other baseline methods using a single NVIDIA GeForce RTX 3090 GPU (Table \ref{tab:computation_profiling}). Inference speeds were profiled using pure XLA-compiled device execution,   isolating the algorithmic complexity from Python-side dispatch overhead. While ME-AM requires a slightly larger memory footprint and higher system throughput during the offline training phase (7.73 ms/step vs. 5.04 ms/step for baseline QAM), this is a direct consequence of optimizing the enriched MaxEnt Adjoint Matching objective, specifically, training the auxiliary score network and the target noise generator. Crucially, these auxiliary networks act purely as training-time regularizers and are completely severed from the computational graph prior to deployment. Consequently, the active inference architecture of ME-AM is identical to standard flow-matching policies, allowing it to achieve significant performance gains while matching the baseline flow-matching inference latency of 0.14 ms/step.

\clearpage % Start a completely fresh page

\section{Full Results} % Your title goes here!

\input{./full_table.tex} 

\begin{figure}[H] % Changed from [p] to [H] to force it directly under the table
  \centering
  \noindent\makebox[\textwidth]{%
   \includegraphics[width=1.45\linewidth]{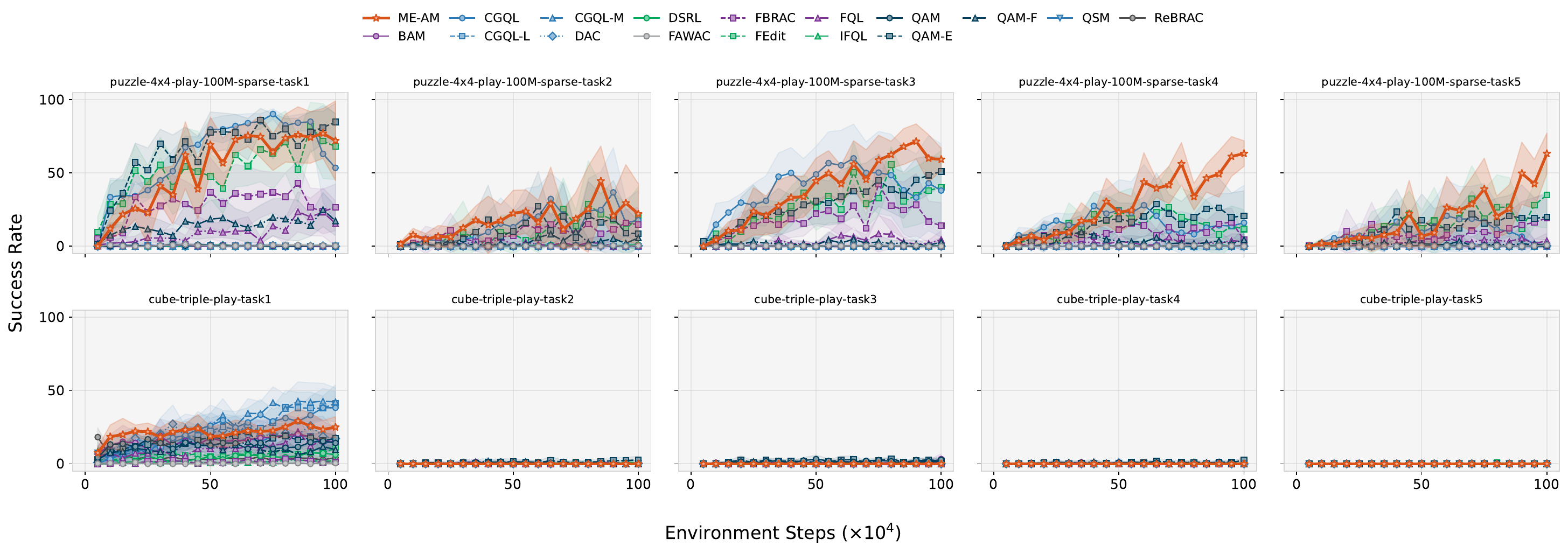}
  }
  \caption{\textbf{Full training curves for the 1M iteration offline phase (8 seeds).}}
  \label{fig:full_training_curves}
\end{figure}

\clearpage % Ensure the next section starts on a new page

\section{Negative Results}
\label{sec:appendix_negative_results}

In developing the Mixture Behavior Prior (Section \ref{sec:method_mixture}), we evaluated alternative architectures to expand the valid data support. We conducted these structural ablations primarily on the \texttt{p44} domain, as its dispersed data manifold   necessitates geometric expansion. We document these observations to provide insight into the empirical sensitivities of continuous-time offline RL.

\paragraph{Naive Stochastic Edits.}
We initially tested unstructured exploration by applying Ornstein-Uhlenbeck \citep{uhlenbeck1930theory, lillicrap2019continuouscontroldeepreinforcement} processes and uniform random noise, evaluating these both as residual edits added to the dataset actions and as independent, full-action samples across the bounds. Empirically, all of these variants led to extreme performance variance across different random seeds, with the critic collapsing in the majority of runs. Specifically, we observed the critic frequently degenerating to predict a constant minimal value ($-100.0$) across all states. This suggests that in fragmented, sparse-reward landscapes, undirected noise overwhelmingly samples zero-reward regions. Consequently, the critic fails to bootstrap a meaningful gradient.

\paragraph{Derivative-Free Candidate Selection.}
We evaluated a derivative-free strategy by generating multiple candidate actions for a given state (again testing both residual noise expansions and absolute action sampling), evaluating their Q-values, and selecting the candidate with the highest predicted return. Our intuition was that this ``best-of-$N$'' selection might circumvent the risk of gradient-based optimization collapsing into bad local minima. However, we observed that performance remained sub-optimal and suffered from similar critic collapse issues. We posit that in high-dimensional action spaces with structural voids, the probability of undirected sampling successfully guessing a valid, high-reward out-of-distribution mode remains extremely low. 

\paragraph{The Necessity of Local Conditioning.}
When optimizing the auxiliary actor $\pi_\omega$, we evaluated a state-only parameterization, $\pi_\omega(\cdot | s)$, to predict a global high-reward coordinate. We observed that this approach often resulted in mode collapse. By withholding the underlying dataset action $a_{\text{data}}$, the actor appeared to predict a single, unimodal target for a given state, which can severely degrade the expressivity of the downstream flow model. This suggests that conditioning on the concatenation, $\pi_\omega(\cdot | s, a_{\text{data}})$, is critical for forcing the network to learn \textit{local} geometric expansions anchored around existing data clusters.

%\newpage
%\input{checklist}

\end{document}

%% file: figs/method_overview.tex
\begin{tikzpicture}[
    every node/.style={scale=1.0},
    pathline/.style={->, ultra thick, draw=black!80, smooth},
    editline/.style={->, ultra thick, draw=black!80, dashed}
]

% ==========================================
% DEFINING CUSTOM PAPER COLORS
% ==========================================
\definecolor{midnightblue}{HTML}{003F5C}
\definecolor{burntorange}{HTML}{D95319}

% ==========================================
% DEFINING SHAPE PLOTS (FOR MULTI-MODALITY)
% ==========================================
% Custom command to draw the multimodal 'flower' shape
\newcommand{\multimodalshape}[1]{
    plot[domain=0:360, samples=100, smooth cycle] (\x: {0.8 + 0.2*cos(4*\x)})
}

% ==========================================
% THE ACTION SPACE GEOMETRY (\Omega)
% ==========================================
% Shifted to the right (starts at x=2.5) to keep it disjoint from initial noise
\filldraw[fill=gray!10, draw=gray!30, thick, rounded corners=10mm] (2.5, -4.5) rectangle (12.5, 4.5);
\node[text=gray!70, font=\LARGE\bfseries] at (11.8, 3.8) {$\Omega$};

% ==========================================
% MIDDLE SEPARATOR
% ==========================================
\draw[ultra thick, dashed, darkgray!60] (-1.5, 0) -- (13.0, 0);

% ==========================================
% TOP BRANCH: QAM & QAM-Edit (Baseline)
% ==========================================
\node[text=black!60, font=\Large\bfseries, anchor=west] at (2.8, 4.0) {(a) QAM \& QAM-Edit};

% Initial Noise Distribution (p_0) - Top
\begin{scope}[shift={(0, 2.0)}]
    % Gaussian fading out to white using midnight blue
    \shade[inner color=midnightblue!70, outer color=white, draw=gray!40, thick] (0,0) circle (1.2);
    \node[text=black!90, font=\huge\bfseries] at (0, 0) {$p_0$};
\end{scope}

% Behavioral Density (\pi_\beta) 
\begin{scope}[shift={(5.5, 2.0)}]
    \shade[inner color=midnightblue, outer color=midnightblue!20, draw=midnightblue, thick] \multimodalshape{0};
    \node[text=white, font=\Large\bfseries] at (0,0) {$\pi_\beta$};
\end{scope}

% Edit Target Density (Orange) 
\begin{scope}[shift={(10.5, 2.0)}]
    \shade[inner color=burntorange, outer color=burntorange!20, draw=burntorange, thick] \multimodalshape{0};
    \node[text=white, font=\Large\bfseries] at (0,0) {$\pi^*$};
    \node[text=burntorange, font=\large\bfseries] at (0, 1.3) {High $Q_\phi$};
\end{scope}

% The Trapped QAM Paths (Two arrows to show flow)
\draw[pathline] (1.1, 2.5) to[out=15, in=165] (4.6, 2.5);
\draw[pathline] (1.1, 1.5) to[out=-15, in=-165] (4.6, 1.5);
\node[above, font=\large\bfseries, text=black!80] at (3.0, 2.7) {$v_{\text{QAM}}$};

% The Disjointed Gaussian Edit 
\draw[editline] (6.5, 2.0) -- node[above, text=black!80, font=\Large\bfseries] {$\Delta a$} (9.5, 2.0);

% ==========================================
% BOTTOM BRANCH: ME-AM (Ours)
% ==========================================
\node[text=midnightblue, font=\Large\bfseries, anchor=west] at (2.8, -4.0) {(b) ME-AM (Ours)};

% Initial Noise Distribution (p_0) - Bottom
\begin{scope}[shift={(0, -2.0)}]
    \shade[inner color=midnightblue!70, outer color=white, draw=gray!40, thick] (0,0) circle (1.2);
    \node[text=black!90, font=\huge\bfseries] at (0, 0) {$p_0$};
\end{scope}

% ME-AM Expanded Support (\pi_mix) - Stretched Flower
\begin{scope}[shift={(8.0, -2.0)}, xscale=3.2, yscale=1.15]
    % Increased scale horizontally and vertically so it comfortably engulfs the right side
    \shade[left color=midnightblue!80, right color=midnightblue!80, middle color=midnightblue!15, draw=midnightblue, thick, dashed] \multimodalshape{0};
\end{scope}
% Moved \pi_mix to sit clearly inside the left-middle portion of the stretched shape
\node[text=white, font=\Large\bfseries] at (6.0, -2.0) {$\pi_{\text{mix}}$};

% Optimal Boltzmann Density (\pi^*) seamlessly embedded in the stretched structure
\begin{scope}[shift={(10.5, -2.0)}]
    \shade[inner color=burntorange, outer color=burntorange!20, draw=burntorange, thick] \multimodalshape{0};
    \node[text=white, font=\Large\bfseries] at (0,0) {$\pi^*$};
    \node[text=burntorange, font=\large\bfseries] at (0, -1.5) {High $Q_\phi$};
\end{scope}

% The Unified ME-AM Paths (Two arrows bridging the entire distance)
\draw[pathline] (1.1, -1.5) to[out=15, in=165] (10.0, -1.5);
\draw[pathline] (1.1, -2.5) to[out=-15, in=-165] (10.0, -2.5);
\node[below, font=\large\bfseries, text=black!80] at (5.0, -2.6) {$v_{\text{ME-AM}}$};

\end{tikzpicture}

%% file: full_results_agg.tex
\begin{table}[t]
    \centering
    \scalebox{0.75}{
    \begin{tabular}{clccc}
\toprule
 & & \texttt{p44} & \texttt{ct} & \texttt{all} \\
 & & \tiny\texttt{5 tasks} & \tiny\texttt{5 tasks} & \tiny\texttt{10 tasks} \\
\midrule
\multirow{2}{*}{\textsc{{Gaussian}}} 
& \texttt{ReBRAC} 
& \phantom{00}$\overset{\phantom{0}\scaleto{[0,0]}{3pt}}{\phantom{0}\phantom{0}0}$\phantom{00} 
& \phantom{00}$\overset{\phantom{0}\scaleto{[0,2]}{3pt}}{\phantom{0}\phantom{0}1}$\phantom{00} 
& \phantom{00}$\overset{\phantom{0}\scaleto{[0,1]}{3pt}}{\phantom{0}\phantom{0}0}$\phantom{00} \\
& \texttt{FAWAC} 
& \phantom{00}$\overset{\phantom{0}\scaleto{[0,1]}{3pt}}{\phantom{0}\phantom{0}0}$\phantom{00} 
& \phantom{00}$\overset{\phantom{0}\scaleto{[0,0]}{3pt}}{\phantom{0}\phantom{0}0}$\phantom{00} 
& \phantom{00}$\overset{\phantom{0}\scaleto{[0,0]}{3pt}}{\phantom{0}\phantom{0}0}$\phantom{00} \\
\midrule
\multirow{3}{*}{\textsc{{Backprop}}} 
& \texttt{FQL} 
& \phantom{00}$\overset{\phantom{0}\scaleto{[2,9]}{3pt}}{\phantom{0}\phantom{0}5}$\phantom{00} 
& \phantom{00}$\overset{\phantom{0}\scaleto{[1,5]}{3pt}}{\phantom{0}\phantom{0}3}$\phantom{00} 
& \phantom{00}$\overset{\phantom{0}\scaleto{[2,6]}{3pt}}{\phantom{0}\phantom{0}4}$\phantom{00} \\
& \texttt{FEdit} 
& \phantom{00}$\overset{\phantom{0}\scaleto{[25,45]}{3pt}}{\phantom{0}35}$\phantom{00} 
& \phantom{00}$\overset{\phantom{0}\scaleto{[1,4]}{3pt}}{\phantom{0}\phantom{0}3}$\phantom{00} 
& \phantom{00}$\overset{\phantom{0}\scaleto{[13,25]}{3pt}}{\phantom{0}19}$\phantom{00} \\
& \texttt{FBRAC} 
& \phantom{00}$\overset{\phantom{0}\scaleto{[10,21]}{3pt}}{\phantom{0}16}$\phantom{00} 
& \phantom{00}$\overset{\phantom{0}\scaleto{[0,1]}{3pt}}{\phantom{0}\phantom{0}0}$\phantom{00} 
& \phantom{00}$\overset{\phantom{0}\scaleto{[5,11]}{3pt}}{\phantom{0}\phantom{0}8}$\phantom{00} \\
\midrule
\multirow{5}{*}{\textsc{{Guidance}}} 
& \texttt{CGQL} 
& \phantom{00}$\overset{\phantom{0}\scaleto{[14,36]}{3pt}}{\phantom{0}25}$\phantom{00} 
& \phantom{00}$\overset{\phantom{0}\scaleto{[3,13]}{3pt}}{\phantom{0}\phantom{0}\mathbf{8}}$\phantom{00} 
& \phantom{00}$\overset{\phantom{0}\scaleto{[10,23]}{3pt}}{\phantom{0}16}$\phantom{00} \\
& \texttt{CGQL-L} 
& \phantom{00}$\overset{\phantom{0}\scaleto{[0,0]}{3pt}}{\phantom{0}\phantom{0}0}$\phantom{00} 
& \phantom{00}$\overset{\phantom{0}\scaleto{[3,14]}{3pt}}{\phantom{0}\phantom{0}\mathbf{8}}$\phantom{00} 
& \phantom{00}$\overset{\phantom{0}\scaleto{[1,7]}{3pt}}{\phantom{0}\phantom{0}4}$\phantom{00} \\
& \texttt{CGQL-M} 
& \phantom{00}$\overset{\phantom{0}\scaleto{[0,0]}{3pt}}{\phantom{0}\phantom{0}0}$\phantom{00} 
& \phantom{00}$\overset{\phantom{0}\scaleto{[3,14]}{3pt}}{\phantom{0}\phantom{0}\mathbf{8}}$\phantom{00} 
& \phantom{00}$\overset{\phantom{0}\scaleto{[1,7]}{3pt}}{\phantom{0}\phantom{0}4}$\phantom{00} \\
& \texttt{DAC} 
& \phantom{00}$\overset{\phantom{0}\scaleto{[0,0]}{3pt}}{\phantom{0}\phantom{0}0}$\phantom{00} 
& \phantom{00}$\overset{\phantom{0}\scaleto{[1,6]}{3pt}}{\phantom{0}\phantom{0}4}$\phantom{00} 
& \phantom{00}$\overset{\phantom{0}\scaleto{[1,3]}{3pt}}{\phantom{0}\phantom{0}2}$\phantom{00} \\
& \texttt{QSM} 
& \phantom{00}$\overset{\phantom{0}\scaleto{[0,0]}{3pt}}{\phantom{0}\phantom{0}0}$\phantom{00} 
& \phantom{00}$\overset{\phantom{0}\scaleto{[1,5]}{3pt}}{\phantom{0}\phantom{0}3}$\phantom{00} 
& \phantom{00}$\overset{\phantom{0}\scaleto{[1,3]}{3pt}}{\phantom{0}\phantom{0}2}$\phantom{00} \\
\midrule
\multirow{2}{*}{\textsc{{Post-processing}}} 
& \texttt{IFQL} 
& \phantom{00}$\overset{\phantom{0}\scaleto{[0,1]}{3pt}}{\phantom{0}\phantom{0}0}$\phantom{00} 
& \phantom{00}$\overset{\phantom{0}\scaleto{[0,1]}{3pt}}{\phantom{0}\phantom{0}0}$\phantom{00} 
& \phantom{00}$\overset{\phantom{0}\scaleto{[0,1]}{3pt}}{\phantom{0}\phantom{0}0}$\phantom{00} \\
& \texttt{DSRL} 
& \phantom{00}$\overset{\phantom{0}\scaleto{[0,0]}{3pt}}{\phantom{0}\phantom{0}0}$\phantom{00} 
& \phantom{00}$\overset{\phantom{0}\scaleto{[0,2]}{3pt}}{\phantom{0}\phantom{0}1}$\phantom{00} 
& \phantom{00}$\overset{\phantom{0}\scaleto{[0,1]}{3pt}}{\phantom{0}\phantom{0}1}$\phantom{00} \\
\midrule
\multirow{5}{*}{\textsc{{Adjoint Matching}}} 
& \texttt{BAM} 
& \phantom{00}$\overset{\phantom{0}\scaleto{[0,0]}{3pt}}{\phantom{0}\phantom{0}0}$\phantom{00} 
& \phantom{00}$\overset{\phantom{0}\scaleto{[1,6]}{3pt}}{\phantom{0}\phantom{0}4}$\phantom{00} 
& \phantom{00}$\overset{\phantom{0}\scaleto{[1,3]}{3pt}}{\phantom{0}\phantom{0}2}$\phantom{00} \\
& \texttt{QAM} 
& \phantom{00}$\overset{\phantom{0}\scaleto{[0,0]}{3pt}}{\phantom{0}\phantom{0}0}$\phantom{00} 
& \phantom{00}$\overset{\phantom{0}\scaleto{[2,5]}{3pt}}{\phantom{0}\phantom{0}4}$\phantom{00} 
& \phantom{00}$\overset{\phantom{0}\scaleto{[1,3]}{3pt}}{\phantom{0}\phantom{0}2}$\phantom{00} \\
& \texttt{QAM-F} 
& \phantom{00}$\overset{\phantom{0}\scaleto{[3,9]}{3pt}}{\phantom{0}\phantom{0}6}$\phantom{00} 
& \phantom{00}$\overset{\phantom{0}\scaleto{[1,4]}{3pt}}{\phantom{0}\phantom{0}2}$\phantom{00} 
& \phantom{00}$\overset{\phantom{0}\scaleto{[3,6]}{3pt}}{\phantom{0}\phantom{0}4}$\phantom{00} \\
& \texttt{QAM-E} 
& \phantom{00}$\overset{\phantom{0}\scaleto{[27,47]}{3pt}}{\phantom{0}37}$\phantom{00} 
& \phantom{00}$\overset{\phantom{0}\scaleto{[2,7]}{3pt}}{\phantom{0}\phantom{0}5}$\phantom{00} 
& \phantom{00}$\overset{\phantom{0}\scaleto{[15,27]}{3pt}}{\phantom{0}21}$\phantom{00} \\
& \texttt{ME-AM (ours)} 
& \phantom{00}$\overset{\phantom{0}\scaleto{[48,64]}{3pt}}{\phantom{0}\mathbf{56}}$\phantom{00} 
& \phantom{00}$\overset{\phantom{0}\scaleto{[2,8]}{3pt}}{\phantom{0}\phantom{0}5}$\phantom{00} 
& \phantom{00}$\overset{\phantom{0}\scaleto{[23,38]}{3pt}}{\phantom{0}\mathbf{30}}$\phantom{00} \\
\bottomrule
\end{tabular}
}
\caption{ \textbf{Aggregate offline RL performance at 1M training steps (8 seeds) across 2 domains (10 tasks).} Comparison of our ME-AM against strong offline RL baselines on sparse-reward domains. For full results on individual tasks, see \cref{tab:sparse-results-tab} in the appendix.} 
\label{tab:agg-sparse-results}
\end{table}

%% file: full_table.tex
\begin{table}[H]
    \centering
    \makebox[\textwidth]{\scalebox{0.55}{
    \begin{tabular}{cl ccccccccccccccccc}
\toprule
 &  & \texttt{ReBRAC} & \texttt{FQL} & \texttt{CGQL} & \texttt{DAC} & \texttt{QSM} & \texttt{FAWAC} & \texttt{IFQL} & \texttt{FEdit} & \texttt{DSRL} & \texttt{CGQL-L} & \texttt{CGQL-M} & \texttt{FBRAC} & \texttt{BAM} & \texttt{QAM} & \texttt{QAM-F} & \texttt{QAM-E} & \texttt{ME-AM (ours)} \\
\midrule
 & \texttt{task1} & \phantom{00}$\overset{\phantom{0}\scaleto{[0,0]}{3pt}}{\phantom{0}\phantom{0}0}$\phantom{00} & \phantom{00}$\overset{\phantom{0}\scaleto{[0,31]}{3pt}}{\phantom{0}16}$\phantom{00} & \phantom{00}$\overset{\phantom{0}\scaleto{[17,90]}{3pt}}{\phantom{0}54}$\phantom{00} & \phantom{00}$\overset{\phantom{0}\scaleto{[0,0]}{3pt}}{\phantom{0}\phantom{0}0}$\phantom{00} & \phantom{00}$\overset{\phantom{0}\scaleto{[0,1]}{3pt}}{\phantom{0}\phantom{0}0}$\phantom{00} & \phantom{00}$\overset{\phantom{0}\scaleto{[0,1]}{3pt}}{\phantom{0}\phantom{0}0}$\phantom{00} & \phantom{00}$\overset{\phantom{0}\scaleto{[0,1]}{3pt}}{\phantom{0}\phantom{0}0}$\phantom{00} & \phantom{00}$\overset{\phantom{0}\scaleto{[45,91]}{3pt}}{\phantom{0}68}$\phantom{00} & \phantom{00}$\overset{\phantom{0}\scaleto{[0,0]}{3pt}}{\phantom{0}\phantom{0}0}$\phantom{00} & \phantom{00}$\overset{\phantom{0}\scaleto{[0,0]}{3pt}}{\phantom{0}\phantom{0}0}$\phantom{00} & \phantom{00}$\overset{\phantom{0}\scaleto{[0,0]}{3pt}}{\phantom{0}\phantom{0}0}$\phantom{00} & \phantom{00}$\overset{\phantom{0}\scaleto{[11,42]}{3pt}}{\phantom{0}26}$\phantom{00} & \phantom{00}$\overset{\phantom{0}\scaleto{[0,0]}{3pt}}{\phantom{0}\phantom{0}0}$\phantom{00} & \phantom{00}$\overset{\phantom{0}\scaleto{[0,0]}{3pt}}{\phantom{0}\phantom{0}0}$\phantom{00} & \phantom{00}$\overset{\phantom{0}\scaleto{[7,27]}{3pt}}{\phantom{0}17}$\phantom{00} & \phantom{00}$\overset{\phantom{0}\scaleto{[72,98]}{3pt}}{\phantom{0}\mathbf{85}}$\phantom{00} & \phantom{00}$\overset{\phantom{0}\scaleto{[45,99]}{3pt}}{\phantom{0}72}$\phantom{00} \\
 & \texttt{task2} & \phantom{00}$\overset{\phantom{0}\scaleto{[0,0]}{3pt}}{\phantom{0}\phantom{0}0}$\phantom{00} & \phantom{00}$\overset{\phantom{0}\scaleto{[0,1]}{3pt}}{\phantom{0}\phantom{0}0}$\phantom{00} & \phantom{00}$\overset{\phantom{0}\scaleto{[0,38]}{3pt}}{\phantom{0}18}$\phantom{00} & \phantom{00}$\overset{\phantom{0}\scaleto{[0,1]}{3pt}}{\phantom{0}\phantom{0}0}$\phantom{00} & \phantom{00}$\overset{\phantom{0}\scaleto{[0,0]}{3pt}}{\phantom{0}\phantom{0}0}$\phantom{00} & \phantom{00}$\overset{\phantom{0}\scaleto{[0,1]}{3pt}}{\phantom{0}\phantom{0}0}$\phantom{00} & \phantom{00}$\overset{\phantom{0}\scaleto{[0,2]}{3pt}}{\phantom{0}\phantom{0}1}$\phantom{00} & \phantom{00}$\overset{\phantom{0}\scaleto{[0,41]}{3pt}}{\phantom{0}20}$\phantom{00} & \phantom{00}$\overset{\phantom{0}\scaleto{[0,2]}{3pt}}{\phantom{0}\phantom{0}1}$\phantom{00} & \phantom{00}$\overset{\phantom{0}\scaleto{[0,0]}{3pt}}{\phantom{0}\phantom{0}0}$\phantom{00} & \phantom{00}$\overset{\phantom{0}\scaleto{[0,0]}{3pt}}{\phantom{0}\phantom{0}0}$\phantom{00} & \phantom{00}$\overset{\phantom{0}\scaleto{[3,26]}{3pt}}{\phantom{0}15}$\phantom{00} & \phantom{00}$\overset{\phantom{0}\scaleto{[0,0]}{3pt}}{\phantom{0}\phantom{0}0}$\phantom{00} & \phantom{00}$\overset{\phantom{0}\scaleto{[0,1]}{3pt}}{\phantom{0}\phantom{0}0}$\phantom{00} & \phantom{00}$\overset{\phantom{0}\scaleto{[0,13]}{3pt}}{\phantom{0}\phantom{0}6}$\phantom{00} & \phantom{00}$\overset{\phantom{0}\scaleto{[0,17]}{3pt}}{\phantom{0}\phantom{0}8}$\phantom{00} & \phantom{00}$\overset{\phantom{0}\scaleto{[1,43]}{3pt}}{\phantom{0}\mathbf{22}}$\phantom{00} \\
 & \texttt{task3} & \phantom{00}$\overset{\phantom{0}\scaleto{[0,0]}{3pt}}{\phantom{0}\phantom{0}0}$\phantom{00} & \phantom{00}$\overset{\phantom{0}\scaleto{[0,11]}{3pt}}{\phantom{0}\phantom{0}4}$\phantom{00} & \phantom{00}$\overset{\phantom{0}\scaleto{[6,70]}{3pt}}{\phantom{0}38}$\phantom{00} & \phantom{00}$\overset{\phantom{0}\scaleto{[0,0]}{3pt}}{\phantom{0}\phantom{0}0}$\phantom{00} & \phantom{00}$\overset{\phantom{0}\scaleto{[0,0]}{3pt}}{\phantom{0}\phantom{0}0}$\phantom{00} & \phantom{00}$\overset{\phantom{0}\scaleto{[0,0]}{3pt}}{\phantom{0}\phantom{0}0}$\phantom{00} & \phantom{00}$\overset{\phantom{0}\scaleto{[0,1]}{3pt}}{\phantom{0}\phantom{0}0}$\phantom{00} & \phantom{00}$\overset{\phantom{0}\scaleto{[15,65]}{3pt}}{\phantom{0}40}$\phantom{00} & \phantom{00}$\overset{\phantom{0}\scaleto{[0,0]}{3pt}}{\phantom{0}\phantom{0}0}$\phantom{00} & \phantom{00}$\overset{\phantom{0}\scaleto{[0,0]}{3pt}}{\phantom{0}\phantom{0}0}$\phantom{00} & \phantom{00}$\overset{\phantom{0}\scaleto{[0,0]}{3pt}}{\phantom{0}\phantom{0}0}$\phantom{00} & \phantom{00}$\overset{\phantom{0}\scaleto{[0,30]}{3pt}}{\phantom{0}14}$\phantom{00} & \phantom{00}$\overset{\phantom{0}\scaleto{[0,0]}{3pt}}{\phantom{0}\phantom{0}0}$\phantom{00} & \phantom{00}$\overset{\phantom{0}\scaleto{[0,0]}{3pt}}{\phantom{0}\phantom{0}0}$\phantom{00} & \phantom{00}$\overset{\phantom{0}\scaleto{[0,6]}{3pt}}{\phantom{0}\phantom{0}3}$\phantom{00} & \phantom{00}$\overset{\phantom{0}\scaleto{[41,61]}{3pt}}{\phantom{0}51}$\phantom{00} & \phantom{00}$\overset{\phantom{0}\scaleto{[52,67]}{3pt}}{\phantom{0}\mathbf{59}}$\phantom{00} \\
 & \texttt{task4} & \phantom{00}$\overset{\phantom{0}\scaleto{[0,0]}{3pt}}{\phantom{0}\phantom{0}0}$\phantom{00} & \phantom{00}$\overset{\phantom{0}\scaleto{[0,8]}{3pt}}{\phantom{0}\phantom{0}3}$\phantom{00} & \phantom{00}$\overset{\phantom{0}\scaleto{[0,42]}{3pt}}{\phantom{0}16}$\phantom{00} & \phantom{00}$\overset{\phantom{0}\scaleto{[0,0]}{3pt}}{\phantom{0}\phantom{0}0}$\phantom{00} & \phantom{00}$\overset{\phantom{0}\scaleto{[0,0]}{3pt}}{\phantom{0}\phantom{0}0}$\phantom{00} & \phantom{00}$\overset{\phantom{0}\scaleto{[0,0]}{3pt}}{\phantom{0}\phantom{0}0}$\phantom{00} & \phantom{00}$\overset{\phantom{0}\scaleto{[0,0]}{3pt}}{\phantom{0}\phantom{0}0}$\phantom{00} & \phantom{00}$\overset{\phantom{0}\scaleto{[0,25]}{3pt}}{\phantom{0}12}$\phantom{00} & \phantom{00}$\overset{\phantom{0}\scaleto{[0,0]}{3pt}}{\phantom{0}\phantom{0}0}$\phantom{00} & \phantom{00}$\overset{\phantom{0}\scaleto{[0,0]}{3pt}}{\phantom{0}\phantom{0}0}$\phantom{00} & \phantom{00}$\overset{\phantom{0}\scaleto{[0,0]}{3pt}}{\phantom{0}\phantom{0}0}$\phantom{00} & \phantom{00}$\overset{\phantom{0}\scaleto{[1,8]}{3pt}}{\phantom{0}\phantom{0}4}$\phantom{00} & \phantom{00}$\overset{\phantom{0}\scaleto{[0,1]}{3pt}}{\phantom{0}\phantom{0}0}$\phantom{00} & \phantom{00}$\overset{\phantom{0}\scaleto{[0,1]}{3pt}}{\phantom{0}\phantom{0}0}$\phantom{00} & \phantom{00}$\overset{\phantom{0}\scaleto{[0,12]}{3pt}}{\phantom{0}\phantom{0}4}$\phantom{00} & \phantom{00}$\overset{\phantom{0}\scaleto{[4,38]}{3pt}}{\phantom{0}21}$\phantom{00} & \phantom{00}$\overset{\phantom{0}\scaleto{[55,72]}{3pt}}{\phantom{0}\mathbf{63}}$\phantom{00} \\
 & \texttt{task5} & \phantom{00}$\overset{\phantom{0}\scaleto{[0,0]}{3pt}}{\phantom{0}\phantom{0}0}$\phantom{00} & \phantom{00}$\overset{\phantom{0}\scaleto{[0,9]}{3pt}}{\phantom{0}\phantom{0}4}$\phantom{00} & \phantom{00}$\overset{\phantom{0}\scaleto{[0,1]}{3pt}}{\phantom{0}\phantom{0}0}$\phantom{00} & \phantom{00}$\overset{\phantom{0}\scaleto{[0,1]}{3pt}}{\phantom{0}\phantom{0}0}$\phantom{00} & \phantom{00}$\overset{\phantom{0}\scaleto{[0,0]}{3pt}}{\phantom{0}\phantom{0}0}$\phantom{00} & \phantom{00}$\overset{\phantom{0}\scaleto{[0,2]}{3pt}}{\phantom{0}\phantom{0}1}$\phantom{00} & \phantom{00}$\overset{\phantom{0}\scaleto{[0,1]}{3pt}}{\phantom{0}\phantom{0}0}$\phantom{00} & \phantom{00}$\overset{\phantom{0}\scaleto{[13,57]}{3pt}}{\phantom{0}35}$\phantom{00} & \phantom{00}$\overset{\phantom{0}\scaleto{[0,0]}{3pt}}{\phantom{0}\phantom{0}0}$\phantom{00} & \phantom{00}$\overset{\phantom{0}\scaleto{[0,0]}{3pt}}{\phantom{0}\phantom{0}0}$\phantom{00} & \phantom{00}$\overset{\phantom{0}\scaleto{[0,0]}{3pt}}{\phantom{0}\phantom{0}0}$\phantom{00} & \phantom{00}$\overset{\phantom{0}\scaleto{[2,37]}{3pt}}{\phantom{0}19}$\phantom{00} & \phantom{00}$\overset{\phantom{0}\scaleto{[0,0]}{3pt}}{\phantom{0}\phantom{0}0}$\phantom{00} & \phantom{00}$\overset{\phantom{0}\scaleto{[0,1]}{3pt}}{\phantom{0}\phantom{0}0}$\phantom{00} & \phantom{00}$\overset{\phantom{0}\scaleto{[0,4]}{3pt}}{\phantom{0}\phantom{0}1}$\phantom{00} & \phantom{00}$\overset{\phantom{0}\scaleto{[4,36]}{3pt}}{\phantom{0}20}$\phantom{00} & \phantom{00}$\overset{\phantom{0}\scaleto{[49,77]}{3pt}}{\phantom{0}\mathbf{63}}$\phantom{00} \\
\multirow{-6}{*}{\texttt{puzzle-4x4-100M-sparse}} & \texttt{agg. (5 tasks)} & \phantom{00}$\overset{\phantom{0}\scaleto{[0,0]}{3pt}}{\phantom{0}\phantom{0}0}$\phantom{00} & \phantom{00}$\overset{\phantom{0}\scaleto{[2,9]}{3pt}}{\phantom{0}\phantom{0}5}$\phantom{00} & \phantom{00}$\overset{\phantom{0}\scaleto{[14,36]}{3pt}}{\phantom{0}25}$\phantom{00} & \phantom{00}$\overset{\phantom{0}\scaleto{[0,0]}{3pt}}{\phantom{0}\phantom{0}0}$\phantom{00} & \phantom{00}$\overset{\phantom{0}\scaleto{[0,0]}{3pt}}{\phantom{0}\phantom{0}0}$\phantom{00} & \phantom{00}$\overset{\phantom{0}\scaleto{[0,1]}{3pt}}{\phantom{0}\phantom{0}0}$\phantom{00} & \phantom{00}$\overset{\phantom{0}\scaleto{[0,1]}{3pt}}{\phantom{0}\phantom{0}0}$\phantom{00} & \phantom{00}$\overset{\phantom{0}\scaleto{[25,45]}{3pt}}{\phantom{0}35}$\phantom{00} & \phantom{00}$\overset{\phantom{0}\scaleto{[0,0]}{3pt}}{\phantom{0}\phantom{0}0}$\phantom{00} & \phantom{00}$\overset{\phantom{0}\scaleto{[0,0]}{3pt}}{\phantom{0}\phantom{0}0}$\phantom{00} & \phantom{00}$\overset{\phantom{0}\scaleto{[0,0]}{3pt}}{\phantom{0}\phantom{0}0}$\phantom{00} & \phantom{00}$\overset{\phantom{0}\scaleto{[10,21]}{3pt}}{\phantom{0}16}$\phantom{00} & \phantom{00}$\overset{\phantom{0}\scaleto{[0,0]}{3pt}}{\phantom{0}\phantom{0}0}$\phantom{00} & \phantom{00}$\overset{\phantom{0}\scaleto{[0,0]}{3pt}}{\phantom{0}\phantom{0}0}$\phantom{00} & \phantom{00}$\overset{\phantom{0}\scaleto{[3,9]}{3pt}}{\phantom{0}\phantom{0}6}$\phantom{00} & \phantom{00}$\overset{\phantom{0}\scaleto{[27,47]}{3pt}}{\phantom{0}37}$\phantom{00} & \phantom{00}$\overset{\phantom{0}\scaleto{[48,64]}{3pt}}{\phantom{0}\mathbf{56}}$\phantom{00} \\
\midrule
 & \texttt{task1} & \phantom{00}$\overset{\phantom{0}\scaleto{[2,7]}{3pt}}{\phantom{0}\phantom{0}4}$\phantom{00} & \phantom{00}$\overset{\phantom{0}\scaleto{[9,22]}{3pt}}{\phantom{0}15}$\phantom{00} & \phantom{00}$\overset{\phantom{0}\scaleto{[27,50]}{3pt}}{\phantom{0}38}$\phantom{00} & \phantom{00}$\overset{\phantom{0}\scaleto{[12,23]}{3pt}}{\phantom{0}17}$\phantom{00} & \phantom{00}$\overset{\phantom{0}\scaleto{[12,20]}{3pt}}{\phantom{0}16}$\phantom{00} & \phantom{00}$\overset{\phantom{0}\scaleto{[0,1]}{3pt}}{\phantom{0}\phantom{0}0}$\phantom{00} & \phantom{00}$\overset{\phantom{0}\scaleto{[1,3]}{3pt}}{\phantom{0}\phantom{0}2}$\phantom{00} & \phantom{00}$\overset{\phantom{0}\scaleto{[8,15]}{3pt}}{\phantom{0}11}$\phantom{00} & \phantom{00}$\overset{\phantom{0}\scaleto{[1,11]}{3pt}}{\phantom{0}\phantom{0}6}$\phantom{00} & \phantom{00}$\overset{\phantom{0}\scaleto{[29,54]}{3pt}}{\phantom{0}41}$\phantom{00} & \phantom{00}$\overset{\phantom{0}\scaleto{[33,52]}{3pt}}{\phantom{0}\mathbf{42}}$\phantom{00} & \phantom{00}$\overset{\phantom{0}\scaleto{[0,3]}{3pt}}{\phantom{0}\phantom{0}1}$\phantom{00} & \phantom{00}$\overset{\phantom{0}\scaleto{[6,24]}{3pt}}{\phantom{0}15}$\phantom{00} & \phantom{00}$\overset{\phantom{0}\scaleto{[9,20]}{3pt}}{\phantom{0}14}$\phantom{00} & \phantom{00}$\overset{\phantom{0}\scaleto{[7,13]}{3pt}}{\phantom{0}10}$\phantom{00} & \phantom{00}$\overset{\phantom{0}\scaleto{[8,26]}{3pt}}{\phantom{0}17}$\phantom{00} & \phantom{00}$\overset{\phantom{0}\scaleto{[18,32]}{3pt}}{\phantom{0}25}$\phantom{00} \\
 & \texttt{task2} & \phantom{00}$\overset{\phantom{0}\scaleto{[0,0]}{3pt}}{\phantom{0}\phantom{0}0}$\phantom{00} & \phantom{00}$\overset{\phantom{0}\scaleto{[0,1]}{3pt}}{\phantom{0}\phantom{0}0}$\phantom{00} & \phantom{00}$\overset{\phantom{0}\scaleto{[0,1]}{3pt}}{\phantom{0}\phantom{0}0}$\phantom{00} & \phantom{00}$\overset{\phantom{0}\scaleto{[0,1]}{3pt}}{\phantom{0}\phantom{0}0}$\phantom{00} & \phantom{00}$\overset{\phantom{0}\scaleto{[0,0]}{3pt}}{\phantom{0}\phantom{0}0}$\phantom{00} & \phantom{00}$\overset{\phantom{0}\scaleto{[0,0]}{3pt}}{\phantom{0}\phantom{0}0}$\phantom{00} & \phantom{00}$\overset{\phantom{0}\scaleto{[0,0]}{3pt}}{\phantom{0}\phantom{0}0}$\phantom{00} & \phantom{00}$\overset{\phantom{0}\scaleto{[0,2]}{3pt}}{\phantom{0}\phantom{0}1}$\phantom{00} & \phantom{00}$\overset{\phantom{0}\scaleto{[0,0]}{3pt}}{\phantom{0}\phantom{0}0}$\phantom{00} & \phantom{00}$\overset{\phantom{0}\scaleto{[0,0]}{3pt}}{\phantom{0}\phantom{0}0}$\phantom{00} & \phantom{00}$\overset{\phantom{0}\scaleto{[0,1]}{3pt}}{\phantom{0}\phantom{0}0}$\phantom{00} & \phantom{00}$\overset{\phantom{0}\scaleto{[0,0]}{3pt}}{\phantom{0}\phantom{0}0}$\phantom{00} & \phantom{00}$\overset{\phantom{0}\scaleto{[0,1]}{3pt}}{\phantom{0}\phantom{0}0}$\phantom{00} & \phantom{00}$\overset{\phantom{0}\scaleto{[0,1]}{3pt}}{\phantom{0}\phantom{0}0}$\phantom{00} & \phantom{00}$\overset{\phantom{0}\scaleto{[0,1]}{3pt}}{\phantom{0}\phantom{0}0}$\phantom{00} & \phantom{00}$\overset{\phantom{0}\scaleto{[0,5]}{3pt}}{\phantom{0}\phantom{0}\mathbf{3}}$\phantom{00} & \phantom{00}$\overset{\phantom{0}\scaleto{[0,0]}{3pt}}{\phantom{0}\phantom{0}0}$\phantom{00} \\
 & \texttt{task3} & \phantom{00}$\overset{\phantom{0}\scaleto{[0,1]}{3pt}}{\phantom{0}\phantom{0}0}$\phantom{00} & \phantom{00}$\overset{\phantom{0}\scaleto{[0,2]}{3pt}}{\phantom{0}\phantom{0}1}$\phantom{00} & \phantom{00}$\overset{\phantom{0}\scaleto{[0,2]}{3pt}}{\phantom{0}\phantom{0}1}$\phantom{00} & \phantom{00}$\overset{\phantom{0}\scaleto{[0,0]}{3pt}}{\phantom{0}\phantom{0}0}$\phantom{00} & \phantom{00}$\overset{\phantom{0}\scaleto{[0,1]}{3pt}}{\phantom{0}\phantom{0}0}$\phantom{00} & \phantom{00}$\overset{\phantom{0}\scaleto{[0,0]}{3pt}}{\phantom{0}\phantom{0}0}$\phantom{00} & \phantom{00}$\overset{\phantom{0}\scaleto{[0,0]}{3pt}}{\phantom{0}\phantom{0}0}$\phantom{00} & \phantom{00}$\overset{\phantom{0}\scaleto{[0,1]}{3pt}}{\phantom{0}\phantom{0}0}$\phantom{00} & \phantom{00}$\overset{\phantom{0}\scaleto{[0,1]}{3pt}}{\phantom{0}\phantom{0}0}$\phantom{00} & \phantom{00}$\overset{\phantom{0}\scaleto{[0,0]}{3pt}}{\phantom{0}\phantom{0}0}$\phantom{00} & \phantom{00}$\overset{\phantom{0}\scaleto{[0,0]}{3pt}}{\phantom{0}\phantom{0}0}$\phantom{00} & \phantom{00}$\overset{\phantom{0}\scaleto{[0,0]}{3pt}}{\phantom{0}\phantom{0}0}$\phantom{00} & \phantom{00}$\overset{\phantom{0}\scaleto{[0,7]}{3pt}}{\phantom{0}\phantom{0}\mathbf{3}}$\phantom{00} & \phantom{00}$\overset{\phantom{0}\scaleto{[2,4]}{3pt}}{\phantom{0}\phantom{0}\mathbf{3}}$\phantom{00} & \phantom{00}$\overset{\phantom{0}\scaleto{[0,4]}{3pt}}{\phantom{0}\phantom{0}2}$\phantom{00} & \phantom{00}$\overset{\phantom{0}\scaleto{[0,4]}{3pt}}{\phantom{0}\phantom{0}2}$\phantom{00} & \phantom{00}$\overset{\phantom{0}\scaleto{[0,0]}{3pt}}{\phantom{0}\phantom{0}0}$\phantom{00} \\
 & \texttt{task4} & \phantom{00}$\overset{\phantom{0}\scaleto{[0,0]}{3pt}}{\phantom{0}\phantom{0}0}$\phantom{00} & \phantom{00}$\overset{\phantom{0}\scaleto{[0,0]}{3pt}}{\phantom{0}\phantom{0}0}$\phantom{00} & \phantom{00}$\overset{\phantom{0}\scaleto{[0,0]}{3pt}}{\phantom{0}\phantom{0}0}$\phantom{00} & \phantom{00}$\overset{\phantom{0}\scaleto{[0,0]}{3pt}}{\phantom{0}\phantom{0}0}$\phantom{00} & \phantom{00}$\overset{\phantom{0}\scaleto{[0,0]}{3pt}}{\phantom{0}\phantom{0}0}$\phantom{00} & \phantom{00}$\overset{\phantom{0}\scaleto{[0,0]}{3pt}}{\phantom{0}\phantom{0}0}$\phantom{00} & \phantom{00}$\overset{\phantom{0}\scaleto{[0,0]}{3pt}}{\phantom{0}\phantom{0}0}$\phantom{00} & \phantom{00}$\overset{\phantom{0}\scaleto{[0,1]}{3pt}}{\phantom{0}\phantom{0}0}$\phantom{00} & \phantom{00}$\overset{\phantom{0}\scaleto{[0,0]}{3pt}}{\phantom{0}\phantom{0}0}$\phantom{00} & \phantom{00}$\overset{\phantom{0}\scaleto{[0,0]}{3pt}}{\phantom{0}\phantom{0}0}$\phantom{00} & \phantom{00}$\overset{\phantom{0}\scaleto{[0,0]}{3pt}}{\phantom{0}\phantom{0}0}$\phantom{00} & \phantom{00}$\overset{\phantom{0}\scaleto{[0,0]}{3pt}}{\phantom{0}\phantom{0}0}$\phantom{00} & \phantom{00}$\overset{\phantom{0}\scaleto{[0,0]}{3pt}}{\phantom{0}\phantom{0}0}$\phantom{00} & \phantom{00}$\overset{\phantom{0}\scaleto{[0,1]}{3pt}}{\phantom{0}\phantom{0}0}$\phantom{00} & \phantom{00}$\overset{\phantom{0}\scaleto{[0,0]}{3pt}}{\phantom{0}\phantom{0}0}$\phantom{00} & \phantom{00}$\overset{\phantom{0}\scaleto{[1,4]}{3pt}}{\phantom{0}\phantom{0}\mathbf{2}}$\phantom{00} & \phantom{00}$\overset{\phantom{0}\scaleto{[0,0]}{3pt}}{\phantom{0}\phantom{0}0}$\phantom{00} \\
 & \texttt{task5} & \phantom{00}$\overset{\phantom{0}\scaleto{[0,0]}{3pt}}{\phantom{0}\phantom{0}0}$\phantom{00} & \phantom{00}$\overset{\phantom{0}\scaleto{[0,0]}{3pt}}{\phantom{0}\phantom{0}0}$\phantom{00} & \phantom{00}$\overset{\phantom{0}\scaleto{[0,0]}{3pt}}{\phantom{0}\phantom{0}0}$\phantom{00} & \phantom{00}$\overset{\phantom{0}\scaleto{[0,0]}{3pt}}{\phantom{0}\phantom{0}0}$\phantom{00} & \phantom{00}$\overset{\phantom{0}\scaleto{[0,0]}{3pt}}{\phantom{0}\phantom{0}0}$\phantom{00} & \phantom{00}$\overset{\phantom{0}\scaleto{[0,0]}{3pt}}{\phantom{0}\phantom{0}0}$\phantom{00} & \phantom{00}$\overset{\phantom{0}\scaleto{[0,0]}{3pt}}{\phantom{0}\phantom{0}0}$\phantom{00} & \phantom{00}$\overset{\phantom{0}\scaleto{[0,1]}{3pt}}{\phantom{0}\phantom{0}0}$\phantom{00} & \phantom{00}$\overset{\phantom{0}\scaleto{[0,0]}{3pt}}{\phantom{0}\phantom{0}0}$\phantom{00} & \phantom{00}$\overset{\phantom{0}\scaleto{[0,0]}{3pt}}{\phantom{0}\phantom{0}0}$\phantom{00} & \phantom{00}$\overset{\phantom{0}\scaleto{[0,0]}{3pt}}{\phantom{0}\phantom{0}0}$\phantom{00} & \phantom{00}$\overset{\phantom{0}\scaleto{[0,0]}{3pt}}{\phantom{0}\phantom{0}0}$\phantom{00} & \phantom{00}$\overset{\phantom{0}\scaleto{[0,0]}{3pt}}{\phantom{0}\phantom{0}0}$\phantom{00} & \phantom{00}$\overset{\phantom{0}\scaleto{[0,0]}{3pt}}{\phantom{0}\phantom{0}0}$\phantom{00} & \phantom{00}$\overset{\phantom{0}\scaleto{[0,0]}{3pt}}{\phantom{0}\phantom{0}0}$\phantom{00} & \phantom{00}$\overset{\phantom{0}\scaleto{[0,0]}{3pt}}{\phantom{0}\phantom{0}0}$\phantom{00} & \phantom{00}$\overset{\phantom{0}\scaleto{[0,0]}{3pt}}{\phantom{0}\phantom{0}0}$\phantom{00} \\
\multirow{-6}{*}{\texttt{cube-triple-play}} & \texttt{agg. (5 tasks)} & \phantom{00}$\overset{\phantom{0}\scaleto{[0,2]}{3pt}}{\phantom{0}\phantom{0}1}$\phantom{00} & \phantom{00}$\overset{\phantom{0}\scaleto{[1,5]}{3pt}}{\phantom{0}\phantom{0}3}$\phantom{00} & \phantom{00}$\overset{\phantom{0}\scaleto{[3,13]}{3pt}}{\phantom{0}\phantom{0}\mathbf{8}}$\phantom{00} & \phantom{00}$\overset{\phantom{0}\scaleto{[1,6]}{3pt}}{\phantom{0}\phantom{0}4}$\phantom{00} & \phantom{00}$\overset{\phantom{0}\scaleto{[1,5]}{3pt}}{\phantom{0}\phantom{0}3}$\phantom{00} & \phantom{00}$\overset{\phantom{0}\scaleto{[0,0]}{3pt}}{\phantom{0}\phantom{0}0}$\phantom{00} & \phantom{00}$\overset{\phantom{0}\scaleto{[0,1]}{3pt}}{\phantom{0}\phantom{0}0}$\phantom{00} & \phantom{00}$\overset{\phantom{0}\scaleto{[1,4]}{3pt}}{\phantom{0}\phantom{0}3}$\phantom{00} & \phantom{00}$\overset{\phantom{0}\scaleto{[0,2]}{3pt}}{\phantom{0}\phantom{0}1}$\phantom{00} & \phantom{00}$\overset{\phantom{0}\scaleto{[3,14]}{3pt}}{\phantom{0}\phantom{0}\mathbf{8}}$\phantom{00} & \phantom{00}$\overset{\phantom{0}\scaleto{[3,14]}{3pt}}{\phantom{0}\phantom{0}\mathbf{8}}$\phantom{00} & \phantom{00}$\overset{\phantom{0}\scaleto{[0,1]}{3pt}}{\phantom{0}\phantom{0}0}$\phantom{00} & \phantom{00}$\overset{\phantom{0}\scaleto{[1,6]}{3pt}}{\phantom{0}\phantom{0}4}$\phantom{00} & \phantom{00}$\overset{\phantom{0}\scaleto{[2,5]}{3pt}}{\phantom{0}\phantom{0}4}$\phantom{00} & \phantom{00}$\overset{\phantom{0}\scaleto{[1,4]}{3pt}}{\phantom{0}\phantom{0}2}$\phantom{00} & \phantom{00}$\overset{\phantom{0}\scaleto{[2,7]}{3pt}}{\phantom{0}\phantom{0}5}$\phantom{00} & \phantom{00}$\overset{\phantom{0}\scaleto{[2,8]}{3pt}}{\phantom{0}\phantom{0}5}$\phantom{00} \\
\bottomrule
\end{tabular}}}
    \caption{\textbf{Full offline results at 1M training steps. (8 seeds)}}
    \label{tab:sparse-results-tab}
\end{table}

%% file: neurips_2026.bbl
\begin{thebibliography}{52}
\providecommand{\natexlab}[1]{#1}
\providecommand{\url}[1]{\texttt{#1}}
\expandafter\ifx\csname urlstyle\endcsname\relax
  \providecommand{\doi}[1]{doi: #1}\else
  \providecommand{\doi}{doi: \begingroup \urlstyle{rm}\Url}\fi

\bibitem[Ada et~al.(2024)Ada, Oztop, and Ugur]{srdp_ada2024}
Suzan~Ece Ada, Erhan Oztop, and Emre Ugur.
\newblock Diffusion policies for out-of-distribution generalization in offline reinforcement learning.
\newblock \emph{IEEE Robotics and Automation Letters (RA-L)}, 9:\penalty0 3116--3123, 2024.

\bibitem[Ball et~al.(2023)Ball, Smith, Kostrikov, and Levine]{ball2023efficient}
Philip~J Ball, Laura Smith, Ilya Kostrikov, and Sergey Levine.
\newblock Efficient online reinforcement learning with offline data.
\newblock In \emph{International Conference on Machine Learning}, pages 1577--1594. PMLR, 2023.

\bibitem[Bradbury et~al.(2018)Bradbury, Frostig, Hawkins, Johnson, Katariya, Leary, Maclaurin, Necula, Paszke, Vander{P}las, Wanderman-{M}ilne, and Zhang]{jax2018github}
James Bradbury, Roy Frostig, Peter Hawkins, Matthew~James Johnson, Yash Katariya, Chris Leary, Dougal Maclaurin, George Necula, Adam Paszke, Jake Vander{P}las, Skye Wanderman-{M}ilne, and Qiao Zhang.
\newblock {JAX}: composable transformations of {P}ython+{N}um{P}y programs, 2018.
\newblock URL \url{http://github.com/jax-ml/jax}.

\bibitem[Chen et~al.(2024{\natexlab{a}})Chen, Lu, Wang, Su, and Zhu]{srpo_chen2024}
Huayu Chen, Cheng Lu, Zhengyi Wang, Hang Su, and Jun Zhu.
\newblock Score regularized policy optimization through diffusion behavior.
\newblock In \emph{International Conference on Learning Representations (ICLR)}, 2024{\natexlab{a}}.

\bibitem[Chen et~al.(2018)Chen, Rubanova, Bettencourt, and Duvenaud]{chen2018neural}
Ricky T.~Q. Chen, Yulia Rubanova, Jesse Bettencourt, and David~K Duvenaud.
\newblock Neural ordinary differential equations.
\newblock In \emph{Advances in Neural Information Processing Systems}, volume~31. Curran Associates, Inc., 2018.

\bibitem[Chen et~al.(2024{\natexlab{b}})Chen, Wang, and Zhou]{dtql_chen2024}
Tianyu Chen, Zhendong Wang, and Mingyuan Zhou.
\newblock Diffusion policies creating a trust region for offline reinforcement learning.
\newblock In \emph{Neural Information Processing Systems (NeurIPS)}, 2024{\natexlab{b}}.

\bibitem[De~Santi et~al.()De~Santi, Protopapas, Hsieh, and Krause]{deverifier}
Riccardo De~Santi, Kimon Protopapas, Ya-Ping Hsieh, and Andreas Krause.
\newblock Verifier-constrained flow expansion for discovery beyond the data.
\newblock In \emph{The Fourteenth International Conference on Learning Representations}.

\bibitem[De~Santi et~al.(2025{\natexlab{a}})De~Santi, Vlastelica, Hsieh, Shen, He, and Krause]{de2025flow}
Riccardo De~Santi, Marin Vlastelica, Ya-Ping Hsieh, Zebang Shen, Niao He, and Andreas Krause.
\newblock Flow density control: Generative optimization beyond entropy-regularized fine-tuning.
\newblock \emph{Advances in Neural Information Processing Systems (NeurIPS)}, 2025{\natexlab{a}}.

\bibitem[De~Santi et~al.(2025{\natexlab{b}})De~Santi, Vlastelica, Hsieh, Shen, He, and Krause]{de2025provable}
Riccardo De~Santi, Marin Vlastelica, Ya-Ping Hsieh, Zebang Shen, Niao He, and Andreas Krause.
\newblock Provable maximum entropy manifold exploration via diffusion models.
\newblock \emph{International Conference on Machine Learning (ICML)}, 2025{\natexlab{b}}.

\bibitem[Dhariwal and Nichol(2021)]{dhariwal2021diffusion}
Prafulla Dhariwal and Alexander Nichol.
\newblock Diffusion models beat gans on image synthesis.
\newblock \emph{Advances in neural information processing systems}, 34:\penalty0 8780--8794, 2021.

\bibitem[Ding et~al.(2024)Ding, Hu, Zhang, Ren, Zhang, Yu, Wang, and Shi]{qvpo_ding2024}
Shutong Ding, Ke~Hu, Zhenhao Zhang, Kan Ren, Weinan Zhang, Jingyi Yu, Jingya Wang, and Ye~Shi.
\newblock Diffusion-based reinforcement learning via {Q}-weighted variational policy optimization.
\newblock In \emph{Neural Information Processing Systems (NeurIPS)}, 2024.

\bibitem[Ding and Jin(2024)]{consistencyac_ding2024}
Zihan Ding and Chi Jin.
\newblock Consistency models as a rich and efficient policy class for reinforcement learning.
\newblock In \emph{International Conference on Learning Representations (ICLR)}, 2024.

\bibitem[Domingo-Enrich et~al.(2025)Domingo-Enrich, Drozdzal, Karrer, and Chen]{domingo-enrich2025adjoint}
Carles Domingo-Enrich, Michal Drozdzal, Brian Karrer, and Ricky T.~Q. Chen.
\newblock Adjoint matching: Fine-tuning flow and diffusion generative models with memoryless stochastic optimal control.
\newblock In \emph{The Thirteenth International Conference on Learning Representations}, 2025.
\newblock URL \url{https://openreview.net/forum?id=xQBRrtQM8u}.

\bibitem[Dong et~al.(2026)Dong, Li, Sadigh, and Finn]{dong2026expo}
Perry Dong, Qiyang Li, Dorsa Sadigh, and Chelsea Finn.
\newblock {EXPO}: Stable reinforcement learning with expressive policies.
\newblock In \emph{The Fourteenth International Conference on Learning Representations}, 2026.
\newblock URL \url{https://openreview.net/forum?id=aFjSjkB6CV}.

\bibitem[Fang et~al.(2025)Fang, Liu, Zhang, Wang, and Jing]{fang2025diffusion}
Linjiajie Fang, Ruoxue Liu, Jing Zhang, Wenjia Wang, and Bingyi Jing.
\newblock Diffusion actor-critic: Formulating constrained policy iteration as diffusion noise regression for offline reinforcement learning.
\newblock In \emph{The Thirteenth International Conference on Learning Representations}, 2025.
\newblock URL \url{https://openreview.net/forum?id=ldVkAO09Km}.

\bibitem[Fujimoto et~al.(2019)Fujimoto, Meger, and Precup]{fujimoto2019off}
Scott Fujimoto, David Meger, and Doina Precup.
\newblock Off-policy deep reinforcement learning without exploration.
\newblock In \emph{International conference on machine learning}, pages 2052--2062. PMLR, 2019.

\bibitem[Haarnoja et~al.(2018)Haarnoja, Zhou, Abbeel, and Levine]{haarnoja2018soft}
Tuomas Haarnoja, Aurick Zhou, Pieter Abbeel, and Sergey Levine.
\newblock Soft actor-critic: Off-policy maximum entropy deep reinforcement learning with a stochastic actor.
\newblock In \emph{International conference on machine learning}, pages 1861--1870. Pmlr, 2018.

\bibitem[Hansen-Estruch et~al.(2023)Hansen-Estruch, Kostrikov, Janner, Kuba, and Levine]{hansen2023idql}
Philippe Hansen-Estruch, Ilya Kostrikov, Michael Janner, Jakub~Grudzien Kuba, and Sergey Levine.
\newblock {IDQL}: Implicit {Q}-learning as an actor-critic method with diffusion policies.
\newblock \emph{arXiv preprint arXiv:2304.10573}, 2023.

\bibitem[Hazan et~al.(2019)Hazan, Kakade, Singh, and Van~Soest]{hazan2019maxent}
E.~Hazan, S.~Kakade, K.~Singh, and A.~Van~Soest.
\newblock Provably efficient maximum entropy exploration.
\newblock In \emph{International Conference on Machine Learning}, 2019.

\bibitem[He et~al.(2023)He, Shen, Zhang, Tan, and Wang]{diffcps_he2023}
Longxiang He, Li~Shen, Linrui Zhang, Junbo Tan, and Xueqian Wang.
\newblock {DiffCPS}: Diffusion model based constrained policy search for offline reinforcement learning.
\newblock \emph{ArXiv}, abs/2310.05333, 2023.

\bibitem[Ho et~al.(2020)Ho, Jain, and Abbeel]{ho2020denoising}
Jonathan Ho, Ajay Jain, and Pieter Abbeel.
\newblock Denoising diffusion probabilistic models.
\newblock In \emph{Proceedings of the 34th International Conference on Neural Information Processing Systems}, NIPS '20, Red Hook, NY, USA, 2020. Curran Associates Inc.
\newblock ISBN 9781713829546.

\bibitem[Holderrieth and Erives(2025)]{holderrieth2025introductionflowmatchingdiffusion}
Peter Holderrieth and Ezra Erives.
\newblock An introduction to flow matching and diffusion models, 2025.
\newblock URL \url{https://arxiv.org/abs/2506.02070}.

\bibitem[Hsieh et~al.(2019)Hsieh, Liu, and Cevher]{hsieh2019finding}
Y.-P. Hsieh, C.~Liu, and V.~Cevher.
\newblock Finding mixed nash equilibria of generative adversarial networks.
\newblock In \emph{International Conference on Machine Learning}, pages 2810--2819. PMLR, 2019.

\bibitem[Hu et~al.(2025)Hu, Liao, Xu, Liu, Li, Ie, Fei, and Liu]{hu2025improving}
Xixi Hu, Runlong Liao, Keyang Xu, Bo~Liu, Yeqing Li, Eugene Ie, Hongliang Fei, and Qiang Liu.
\newblock Improving rectified flow with boundary conditions, 2025.
\newblock URL \url{https://arxiv.org/abs/2506.15864}.

\bibitem[Kang et~al.(2023)Kang, Ma, Du, Pang, and Yan]{edp_kang2023}
Bingyi Kang, Xiao Ma, Chao Du, Tianyu Pang, and Shuicheng Yan.
\newblock Efficient diffusion policies for offline reinforcement learning.
\newblock In \emph{Neural Information Processing Systems (NeurIPS)}, 2023.

\bibitem[Kappen(2005)]{kappen2005path}
H~J Kappen.
\newblock Path integrals and symmetry breaking for optimal control theory.
\newblock \emph{Journal of Statistical Mechanics: Theory and Experiment}, 2005\penalty0 (11), nov 2005.

\bibitem[Kostrikov et~al.(2022)Kostrikov, Nair, and Levine]{kostrikov2021offline}
Ilya Kostrikov, Ashvin Nair, and Sergey Levine.
\newblock Offline reinforcement learning with implicit q-learning.
\newblock In \emph{International Conference on Learning Representations}, 2022.
\newblock URL \url{https://openreview.net/forum?id=68n2s9ZJWF8}.

\bibitem[Kumar et~al.(2019)Kumar, Fu, Tucker, and Levine]{kumar2019stabilizing}
Aviral Kumar, Justin Fu, George Tucker, and Sergey Levine.
\newblock \emph{Stabilizing off-policy Q-learning via bootstrapping error reduction}.
\newblock Curran Associates Inc., Red Hook, NY, USA, 2019.

\bibitem[Levine et~al.(2020)Levine, Kumar, Tucker, and Fu]{levine2020offline}
Sergey Levine, Aviral Kumar, George Tucker, and Justin Fu.
\newblock Offline reinforcement learning: Tutorial, review, and perspectives on open problems, 2020.
\newblock URL \url{https://arxiv.org/abs/2005.01643}.

\bibitem[Li and Levine(2026)]{li2026qlearning}
Qiyang Li and Sergey Levine.
\newblock Q-learning with adjoint matching.
\newblock In \emph{The Fourteenth International Conference on Learning Representations}, 2026.
\newblock URL \url{https://openreview.net/forum?id=vd4eNAdtO6}.

\bibitem[Li et~al.(2025)Li, Zhou, and Levine]{li2025reinforcement}
Qiyang Li, Zhiyuan Zhou, and Sergey Levine.
\newblock Reinforcement learning with action chunking.
\newblock In \emph{The Thirty-ninth Annual Conference on Neural Information Processing Systems}, 2025.
\newblock URL \url{https://openreview.net/forum?id=XUks1Y96NR}.

\bibitem[Lillicrap et~al.(2019)Lillicrap, Hunt, Pritzel, Heess, Erez, Tassa, Silver, and Wierstra]{lillicrap2019continuouscontroldeepreinforcement}
Timothy~P. Lillicrap, Jonathan~J. Hunt, Alexander Pritzel, Nicolas Heess, Tom Erez, Yuval Tassa, David Silver, and Daan Wierstra.
\newblock Continuous control with deep reinforcement learning, 2019.
\newblock URL \url{https://arxiv.org/abs/1509.02971}.

\bibitem[Lipman et~al.(2023)Lipman, Chen, Ben-Hamu, Nickel, and Le]{lipman2023flow}
Yaron Lipman, Ricky T.~Q. Chen, Heli Ben-Hamu, Maximilian Nickel, and Matthew Le.
\newblock Flow matching for generative modeling.
\newblock In \emph{The Eleventh International Conference on Learning Representations}, 2023.
\newblock URL \url{https://openreview.net/forum?id=PqvMRDCJT9t}.

\bibitem[Lu et~al.(2023)Lu, Chen, Chen, Su, Li, and Zhu]{qgpo_lu2023}
Cheng Lu, Huayu Chen, Jianfei Chen, Hang Su, Chongxuan Li, and Jun Zhu.
\newblock Contrastive energy prediction for exact energy-guided diffusion sampling in offline reinforcement learning.
\newblock In \emph{International Conference on Machine Learning (ICML)}, 2023.

\bibitem[Nair et~al.(2021)Nair, Dalal, Gupta, and Levine]{nair2020awac}
Ashvin Nair, Murtaza Dalal, Abhishek Gupta, and Sergey Levine.
\newblock {\{}AWAC{\}}: Accelerating online reinforcement learning with offline datasets, 2021.
\newblock URL \url{https://openreview.net/forum?id=OJiM1R3jAtZ}.

\bibitem[Park et~al.(2025{\natexlab{a}})Park, Frans, Eysenbach, and Levine]{park2025ogbench}
Seohong Park, Kevin Frans, Benjamin Eysenbach, and Sergey Levine.
\newblock {OGB}ench: Benchmarking offline goal-conditioned {RL}.
\newblock In \emph{The Thirteenth International Conference on Learning Representations}, 2025{\natexlab{a}}.
\newblock URL \url{https://openreview.net/forum?id=M992mjgKzI}.

\bibitem[Park et~al.(2025{\natexlab{b}})Park, Frans, Mann, Eysenbach, Kumar, and Levine]{park2025horizon}
Seohong Park, Kevin Frans, Deepinder Mann, Benjamin Eysenbach, Aviral Kumar, and Sergey Levine.
\newblock Horizon reduction makes {RL} scalable.
\newblock In \emph{The Thirty-ninth Annual Conference on Neural Information Processing Systems}, 2025{\natexlab{b}}.
\newblock URL \url{https://openreview.net/forum?id=hguaupzLCU}.

\bibitem[Park et~al.(2025{\natexlab{c}})Park, Li, and Levine]{park2025flow}
Seohong Park, Qiyang Li, and Sergey Levine.
\newblock Flow {Q}-learning.
\newblock In \emph{Forty-second International Conference on Machine Learning}, 2025{\natexlab{c}}.
\newblock URL \url{https://openreview.net/forum?id=KVf2SFL1pi}.

\bibitem[Peng et~al.(2019)Peng, Kumar, Zhang, and Levine]{peng2019advantage}
Xue~Bin Peng, Aviral Kumar, Grace Zhang, and Sergey Levine.
\newblock Advantage-weighted regression: Simple and scalable off-policy reinforcement learning, 2019.
\newblock URL \url{https://arxiv.org/abs/1910.00177}.

\bibitem[Peters et~al.(2010)Peters, Mulling, and Altun]{peters2010relative}
Jan Peters, Katharina Mulling, and Yasemin Altun.
\newblock Relative entropy policy search.
\newblock In \emph{Proceedings of the AAAI Conference on Artificial Intelligence}, volume~24, pages 1607--1612, 2010.

\bibitem[Psenka et~al.(2024)Psenka, Escontrela, Abbeel, and Ma]{qsm_psenka2024}
Michael Psenka, Alejandro Escontrela, Pieter Abbeel, and Yi~Ma.
\newblock Learning a diffusion model policy from rewards via {Q}-score matching.
\newblock In \emph{International Conference on Machine Learning (ICML)}, 2024.

\bibitem[Rawlik et~al.(2013)Rawlik, Toussaint, and Vijayakumar]{rawlik2013stochastic}
Konrad Rawlik, Marc Toussaint, and Sethu Vijayakumar.
\newblock On stochastic optimal control and reinforcement learning by approximate inference.
\newblock In \emph{Twenty-Third International Joint Conference on Artificial Intelligence}, 2013.

\bibitem[Ren et~al.(2025)Ren, Lidard, Ankile, Simeonov, Agrawal, Majumdar, Burchfiel, Dai, and Simchowitz]{ren2024diffusion}
Allen~Z. Ren, Justin Lidard, Lars~Lien Ankile, Anthony Simeonov, Pulkit Agrawal, Anirudha Majumdar, Benjamin Burchfiel, Hongkai Dai, and Max Simchowitz.
\newblock Diffusion policy policy optimization.
\newblock In \emph{The Thirteenth International Conference on Learning Representations}, 2025.
\newblock URL \url{https://openreview.net/forum?id=mEpqHvbD2h}.

\bibitem[Song and Ermon(2019)]{song2019generative}
Yang Song and Stefano Ermon.
\newblock \emph{Generative modeling by estimating gradients of the data distribution}.
\newblock Curran Associates Inc., Red Hook, NY, USA, 2019.

\bibitem[Tarasov et~al.(2024)Tarasov, Kurenkov, Nikulin, and Kolesnikov]{tarasov2024revisiting}
Denis Tarasov, Vladislav Kurenkov, Alexander Nikulin, and Sergey Kolesnikov.
\newblock Revisiting the minimalist approach to offline reinforcement learning.
\newblock \emph{Advances in Neural Information Processing Systems}, 36, 2024.

\bibitem[Uhlenbeck and Ornstein(1930)]{uhlenbeck1930theory}
George~E Uhlenbeck and Leonard~S Ornstein.
\newblock On the theory of the brownian motion.
\newblock \emph{Physical review}, 36\penalty0 (5):\penalty0 823, 1930.

\bibitem[Vincent(2011)]{vincent2011connection}
Pascal Vincent.
\newblock A connection between score matching and denoising autoencoders.
\newblock \emph{Neural Computation}, 23:\penalty0 1661--1674, 2011.
\newblock URL \url{https://api.semanticscholar.org/CorpusID:5560643}.

\bibitem[Wagenmaker et~al.(2025)Wagenmaker, Nakamoto, Zhang, Park, Yagoub, Nagabandi, Gupta, and Levine]{wagenmaker2025steering}
Andrew Wagenmaker, Mitsuhiko Nakamoto, Yunchu Zhang, Seohong Park, Waleed Yagoub, Anusha Nagabandi, Abhishek Gupta, and Sergey Levine.
\newblock Steering your diffusion policy with latent space reinforcement learning.
\newblock \emph{Conference on Robot Learning}, 2025.

\bibitem[Wang et~al.(2023)Wang, Hunt, and Zhou]{dql_wang2023}
Zhendong Wang, Jonathan~J Hunt, and Mingyuan Zhou.
\newblock Diffusion policies as an expressive policy class for offline reinforcement learning.
\newblock In \emph{International Conference on Learning Representations (ICLR)}, 2023.

\bibitem[Wang et~al.(2020)Wang, Novikov, Zolna, Merel, Springenberg, Reed, Shahriari, Siegel, Gulcehre, Heess, and de~Freitas]{wang2020critic}
Ziyu Wang, Alexander Novikov, Konrad Zolna, Josh~S Merel, Jost~Tobias Springenberg, Scott~E Reed, Bobak Shahriari, Noah Siegel, Caglar Gulcehre, Nicolas Heess, and Nando de~Freitas.
\newblock Critic regularized regression.
\newblock In H.~Larochelle, M.~Ranzato, R.~Hadsell, M.F. Balcan, and H.~Lin, editors, \emph{Advances in Neural Information Processing Systems}, volume~33, pages 7768--7778. Curran Associates, Inc., 2020.
\newblock URL \url{https://proceedings.neurips.cc/paper_files/paper/2020/file/588cb956d6bbe67078f29f8de420a13d-Paper.pdf}.

\bibitem[Yuan et~al.(2025)Yuan, Mu, Tao, Fang, Zhang, and Su]{yuan2024policy}
Xiu Yuan, Tongzhou Mu, Stone Tao, Yunhao Fang, Mengke Zhang, and Hao Su.
\newblock Policy decorator: Model-agnostic online refinement for large policy model.
\newblock In \emph{The Thirteenth International Conference on Learning Representations}, 2025.
\newblock URL \url{https://openreview.net/forum?id=e5jGTEiJMT}.

\bibitem[Zhang et~al.(2024)Zhang, Luo, Sj{\"o}lund, Sch{\"o}n, and Mattsson]{entropydql_zhang2024}
Ruoqi Zhang, Ziwei Luo, Jens Sj{\"o}lund, Thomas~B Sch{\"o}n, and Per Mattsson.
\newblock Entropy-regularized diffusion policy with {Q}-ensembles for offline reinforcement learning.
\newblock In \emph{Neural Information Processing Systems (NeurIPS)}, 2024.

\end{thebibliography}
